\documentclass{article} 

\usepackage[accepted]{icml2020}
\usepackage{times}

\usepackage{graphicx}

\usepackage{amsthm}
\usepackage{amssymb}
\usepackage{amsmath}
\usepackage{color}

\usepackage{hyperref}
\usepackage{url}

\newtheorem{theorem}{Theorem}
\newtheorem{definition}{Definition}
\newtheorem{lemma}{Lemma}
\newtheorem{corollary}{Corollary}
\newtheorem{assumption}{Assumption}

\def\hW{\hat\cW}

\def\eee#1#2{\mathbb{E}_{#1}\left[#2\right]}

\def\vw{\mathbf{w}}

\def\vx{\mathbf{x}}

\def\vu{\mathbf{u}}
\def\vv{\mathbf{v}}
\def\vg{\mathbf{g}}

\def\ve{\mathbf{e}}

\def\vf{\mathbf{f}}

\def\vbeta{\boldsymbol{\beta}}
\def\valpha{\boldsymbol{\alpha}}

\def\cN{\mathcal{N}}
\def\cW{\mathcal{W}}
\def\cB{\mathcal{B}}
\def\cD{\mathcal{D}}

\def\cR{\mathcal{R}}
\def\cO{\mathcal{O}}

\def\vzero{\boldsymbol{0}}

\def\t{\intercal}

\def\rr{\mathbb{R}}

\def\pr{\mathbb{P}}

\icmltitlerunning{Student Specialization in Deep Rectified Networks With Finite Width and Input Dimension}

\begin{document}

\twocolumn[
\icmltitle{Student Specialization in Deep Rectified Networks With Finite Width and Input Dimension}

\begin{icmlauthorlist}
\icmlauthor{Yuandong Tian}{yuandong}
\end{icmlauthorlist}

\icmlaffiliation{yuandong}{Facebook AI Research}

\icmlcorrespondingauthor{Yuandong Tian}{yuandong@fb.com}

\icmlkeywords{Machine Learning, ICML}

\vskip 0.3in
]

\printAffiliationsAndNotice{}

\begin{abstract}
We consider a deep ReLU / Leaky ReLU student network trained from the output of a fixed teacher network of the same depth, with Stochastic Gradient Descent (SGD). The student network is \emph{over-realized}: at each layer $l$, the number $n_l$ of student nodes is more than that ($m_l$) of teacher. Under mild conditions on dataset and teacher network, we prove that when the gradient is small at every data sample, each teacher node is \emph{specialized} by at least one student node \emph{at the lowest layer}. For two-layer network, such specialization can be achieved by training on any dataset of \emph{polynomial} size $\mathcal{O}( K^{5/2} d^3 \epsilon^{-1})$. until the gradient magnitude drops to $\mathcal{O}(\epsilon/K^{3/2}\sqrt{d})$. Here $d$ is the input dimension, $K = m_1 + n_1$ is the total number of neurons in the lowest layer of teacher and student. Note that we require a specific form of data augmentation and the sample complexity includes the additional data generated from augmentation. To our best knowledge, we are the first to give polynomial sample complexity for student specialization of training two-layer (Leaky) ReLU networks with finite depth and width in teacher-student setting, and finite complexity for the lowest layer specialization in multi-layer case, without parametric assumption of the input (like Gaussian). Our theory suggests that teacher nodes with large fan-out weights get specialized first when the gradient is still large, while others are specialized with small gradient, which suggests inductive bias in training. This shapes the stage of training as empirically observed in multiple previous works. Experiments on synthetic and CIFAR10 verify our findings. The code is released in \url{https://github.com/facebookresearch/luckmatters}. 
\end{abstract}

\section{Introduction}
While Deep Learning has achieved great success in different empirical fields~\citep{alphago,resnet,bert}, it remains an open question how such networks can generalize to new data. As shown by empirical studies (e.g.,~\cite{rethinking}), for deep models, training on real or random labels might lead to very different generalization behaviors. Without any assumption on the dataset, the generalization bound can be vacuous, i.e., the same network with zero training error can either generalize well or perform randomly in the test set. 

One way to impose that is via \emph{teacher-student} setting: given $N$ input samples, a fixed teacher network provides the label for a student to learn. The setting has a long history~\citep{Gardner_1989} and offers multiple benefits. First, while worst-case performance on arbitrary data distributions may not be a good model for real structured dataset and can be hard to analyze, using a teacher network implicitly enforces an inductive bias and could potentially lead to better generalization bound. Second, the existence of teacher is often guaranteed by expressiblility (e.g., even one-hidden layer can fit any function~\citep{hornik1989multilayer}). Finally, a reference network could facilitate and deepen our understanding of the training procedure.

Specialization, i.e., a student node becomes increasingly correlated with a teacher node during training~\citep{saad1996dynamics}, is one important topic in this setup. If all student nodes are specialized to the teacher, then student tends to output the same as the teacher and generalization performance can be expected. Empirically, it has been observed in 2-layer networks~\citep{saad1996dynamics,goldt2019dynamics} and multi-layer networks~\citep{tian2019luck,li2015convergent}, in both synthetic and real dataset. In contrast, theoretical analysis is limited with strong assumptions (e.g., Gaussian inputs, infinite input dimension, local convergence, 2-layer setting, small number of hidden nodes). 

In this paper, we analyze student specialization when both teacher and student are deep ReLU / Leaky ReLU~\cite{maas2013rectifier} networks. Similar to~\citep{goldt2019dynamics}, the student is \emph{over-realized} compared to the teacher: at each layer $l$, the number $n_l$ of student nodes is larger than the number $m_l$ of teacher ($n_l \ge m_l$). Although over-realization is different from \emph{over-parameterization}, i.e., the total number of parameters in the student model is larger than the training set size $N$, over-realization directly correlates with the width of networks and is a measure of over-parameterization. With finite input dimension, we show rigorously that when gradient at each training sample is small (i.e., the interpolation setting as suggested in~\citep{ma2017power,liu2018mass,bassily2018exponential}), student nodes~\emph{specialize} to teacher nodes \emph{at the lowest layer}: \textbf{each teacher node is aligned with at least one student node}. This explains one-to-many mapping between teacher and student nodes and the existence of un-specialized student nodes, as observed empirically in~\citep{saad1996dynamics}.

Our setting is more relaxed than previous works. \textbf{(1)} While statistical mechanics approaches~\citep{saad1996dynamics,goldt2019dynamics,Gardner_1989,aubin2018committee} assume both the training set size $N$ and the input dimension $d$ goes to infinite (i.e., the thermodynamics limits) and assume Gaussian inputs, our analysis allows finite $d$ and $N$, and impose \emph{no} parametric constraints on the input data distribution. \textbf{(2)} While Neural Tangent Kernel~\citep{jacot2018neural,du2018gradient} and mean-field approaches~\citep{mei2018mean} requires infinite (or very large) width, our setting applies to finite width as long as student is slightly over-realized ($n_l \ge m_l$). \textbf{(3)} While recent works~\cite{hu2020optimal} show convergence in classification for teacher-student setting when $N$ grows exponentially w.r.t. number of teacher nodes (including 2-layer case), we address student specialization in regression problems and show polynomial sample complexity in 2-layer case. 

We verify our findings with numerical experiments. For deep ReLU nodes, we show one-to-many specialization and existence of un-specialized nodes at each hidden layer, on both synthetic dataset and CIFAR10. We also perform ablation studies about the effect of student over-realization and how strong teacher nodes learn first compared to the weak ones, as suggested by our theory. 

\section{Related Works}
\textbf{Student-teacher setting}. This setting has a long history~\citep{engel2001statistical,Gardner_1989,saad1996dynamics,saad1995line,freeman1997online,mace1998statistical} and recently gains increasing interest~\citep{goldt2019dynamics,aubin2018committee} in analyzing 2-layered network. The seminar works~\citep{saad1996dynamics,saad1995line} studies 1-hidden layer case from statistical mechanics point of view in which the input dimension goes to infinity, or so-called \emph{thermodynamics limits}. They study symmetric solutions and locally analyze the symmetric breaking behavior and onset of \emph{specialization} of the student nodes towards the teacher. Recent follow-up works~\citep{goldt2019dynamics} makes the analysis rigorous and empirically shows that random initialization and training with SGD indeed gives student specialization in 1-hidden layer case, which is consistent with our experiments. With the same assumption, ~\citep{aubin2018committee} studies phase transition property of specialization in 2-layer networks with small number of hidden nodes using replica formula. In these works, inputs are assumed to be Gaussian and step or Gauss error function is used as nonlinearity. Few works study teacher-student setting with more than two layers.~\citep{allen2018learning} shows the recovery results for 2 and 3 layer networks, with modified SGD, batchsize 1 and heavy over-parameterization. 

In comparison, our work shows that specialization happens around SGD critical points in the lowest layer for deep ReLU networks, without any parametric assumptions of input distribution. 

\textbf{Local minima is Global}. While in deep linear network, all local minima are global~\citep{laurent2018deep, kawaguchi2016deep}, situations are quite complicated with nonlinear activations. While local minima is global when the network has invertible activation function and distinct training samples~\citep{nguyen2017loss,yun2018global} or Leaky ReLU with linear separate input data~\citep{laurent2017multilinear}, multiple works~\citep{du2017gradient, ge2017learning, safran2017spurious,yun2018small} show that in GD case with population or empirical loss, spurious local minima can happen even in two-layered network. Many are specific to two-layer and hard to generalize to multi-layer setting. In contrast, our work brings about a generic formulation for deep ReLU network and gives recovery properties in the student-teacher setting.  

\textbf{Learning very wide networks}. Recent works on Neural Tangent Kernel~\citep{jacot2018neural,du2018gradient,pmlr-v97-allen-zhu19a} show the global convergence of GD for multi-layer networks with infinite width. ~\citep{li2018learning} shows the convergence in one-hidden layer ReLU network using GD/SGD to solution with good generalization, when the input data are assumed to be clustered into classes. Both lines of work assume heavily over-parameterized network, requiring polynomial growth of number of nodes with respect to the number of samples. ~\citep{chizat2018global} shows global convergence of over-parameterized network with optimal transport. ~\citep{tian2019luck} assumes mild over-realization and gives convergence results for 2-layer network when a subset of the student network is close to the teacher. Our work extends it to multilayer cases with much weaker assumptions.

\section{Mathematical Framework}
\label{sec:framework}

\def\cleaky{c_\mathrm{leaky}}

\begin{figure*}
    \centering
    \includegraphics[width=0.85\textwidth]{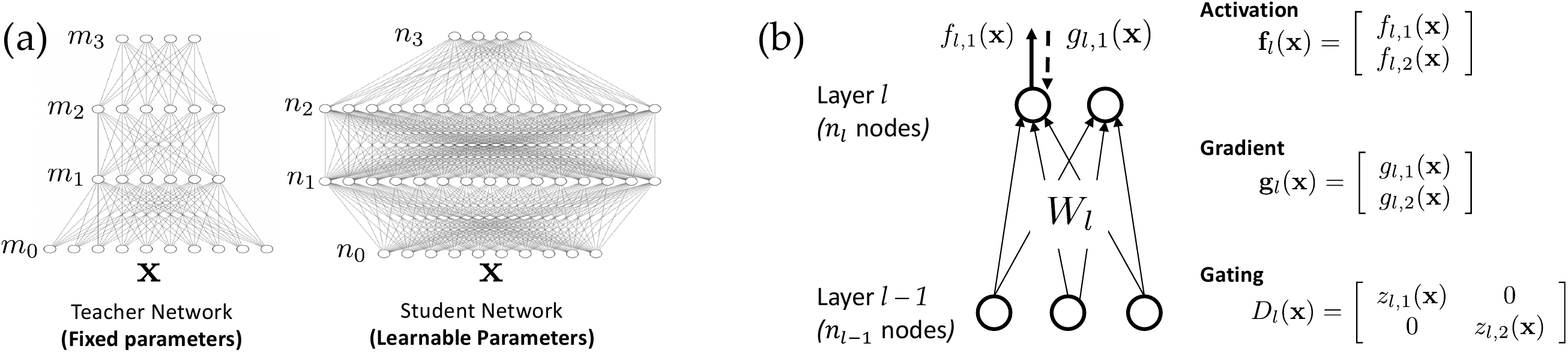}
    \vspace{-0.1in}
    \caption{\small{Problem Setup. \textbf{(a)} Student-teacher setting. The student network learns from the output of a fixed teacher network via stochastic gradient descent (SGD). \textbf{(b)} Notations. All low cases are scalar, bold lower case are column vectors (row vectors are always with a transpose) and upper cases are matrices.}}
    \vspace{-0.1in}
    \label{fig:setup}
\end{figure*}

\textbf{Notation}. Consider a student network and its associated teacher network (Fig.~\ref{fig:setup}(a)). Denote the input as $\vx$. We focus on multi-layered networks with $\sigma(\cdot)$ as ReLU / Leaky ReLU~\cite{maas2013rectifier} nonlinearity. We use the following equality extensively: $\sigma(x) = \mathbb{I}[x \ge 0]x + \mathbb{I}[x < 0] \cleaky x$, where $\mathbb{I}[\cdot]$ is the indicator function and $\cleaky$ is the leaky ReLU constant ($\cleaky=0$ for ReLU). For node $j$, $f_j(\vx)$, $z_j(\vx)$ and $g_j(\vx)$ are its activation, gating function and backpropagated gradient \emph{after the gating}. 

Both teacher and student networks have $L$ layers. The input layer is layer 0 and the topmost layer (layer that is closest to the output) is layer $L$. For layer $l$, let $m_l$ be the number of teacher node while $n_l$ be the number of student node. The weights $W_l\in \rr^{ (n_{l-1} + 1) \times n_l}$ refers to the weight matrix that connects layer $l-1$ to layer $l$ on the student side, with bias terms included. $W_l = [\vw_{l,1}, \vw_{l,2}, \ldots, \vw_{l,n_l}]$ where each $\vw = [\tilde\vw; b]\in \rr^{n_{l-1} + 1}$ is the weight vector. Here $\tilde\vw$ is the weight and $b$ is the bias. 

Without loss of generality, we assume teacher weights $\vw = [\tilde\vw, b] \in \rr^{d+1}$ are \emph{regular}, except for the topmost layer $l=L$.
\begin{definition}[Regular Weight]
A weight vector $\vw = [\tilde\vw, b]$ is \emph{regular} if $\|\tilde\vw\|_2 = 1$.
\end{definition}

Let $\vf_l(\vx) = [f_{l,1}(\vx), \ldots, f_{l,n_l}(\vx), 1]^\t \in \rr^{n_l+1}$ be the activation vector of layer $l$, $D_l(\vx) = \mathrm{diag}[z_{l,1}(\vx), \ldots, z_{l,n_l}(\vx), 1] \in \rr^{(n_l+1)\times (n_l+1)}$ be the diagonal matrix of gating function (for ReLU it is either 0 or 1), and $\vg_l(\vx) = [g_{l,1}(\vx), \ldots, g_{l,n_l}(\vx), 1]^\t \in \rr^{n_l+1}$ be the backpropated gradient vector. The last $1$s are for bias. By definition, $\vf_0(\vx) = \vx \in \rr^{n_0}$ is the input and $m_0 = n_0$. Note that $\vf_l(\vx)$, $\vg_l(\vx)$ and $D_l(\vx)$ are all dependent on $\cW$. For brevity, we often use $\vf_l(\vx)$ rather than $\vf_l(\vx; \cW)$.

All notations with superscript $^*$ are from the teacher, only dependent on the teacher and remains the same throughout the training. For the topmost layer, $D^*_L(\vx) = D_L(\vx) \equiv I_{C\times C}$ since there is no ReLU gating, where $C$ is the dimension of output for both teacher and student. With the notation, gradient descent is:
\begin{equation}
    \dot W_l = \eee{\vx}{\vf_{l-1}(\vx)\vg^\t_l(\vx)} \label{eq:gradient-rule}
\end{equation}
In SGD, the expectation $\eee{\vx}{\cdot}$ is taken over a batch. In GD, it is over the entire dataset.

Let $E_j = \{\vx: f_j(\vx) > 0\}$ the activation region of node $j$ and $\partial E_j = \{\vx: f_j(\vx) = 0\}$ its decision boundary. 

\textbf{MSE Loss}. We use the output of teacher as the supervision:
\begin{equation}
    \min_{\cW} J(\cW) = \frac{1}{2}\eee{\vx}{\|\vf^*_L(\vx) - \vf_L(\vx)\|^2} \label{eq:objective}
\end{equation}

For convenience, we also define the following qualities:
\begin{definition}
\label{def:V-alpha-beta}
Define $V_l \in \rr^{C\times n_l}$ and $V^*_l\in \rr^{C\times m_l}$ in a top-down manner (for top-most layer $V_L = V^*_L := I_{C\times C}$):
\begin{equation}
    V_{l-1} := V_lD_lW^\t_l, \quad V^*_{l-1} := V^*_lD^*_l W^{*\t}_l \label{eq:v-recursive}
\end{equation}
We further define $A_l := V^\t_l V^*_l \in \rr^{n_l\times m_l}$ and $B_l = V_l^\t V_l\in \rr^{n_l\times n_l}$. Each element of $A_l$ is $\alpha^l_{jj'} = \vv^\t_j\vv^*_{j'}$. Similarly, each element of $B_l$ is $\beta^l_{jj'} = \vv^\t_j\vv_{j'}$.
\end{definition}

In this paper, we want to know whether \emph{the student nodes specialize to teacher nodes at the same layers}  during training? We define student node specialization as follows:
\begin{definition}[$\epsilon$-aligned] 
Two nodes $j$ and $j'$ are aligned if their weights $[\tilde\vw_j, b_j]$ and $[\tilde\vw_{j'}, b_{j'}]$ satisfy:
\begin{equation}
\sin\tilde\theta_{jj'} \le \epsilon, \quad\quad |b_j - b_{j'}| \le \epsilon,
\end{equation}
where $\tilde\theta_{jj'}$ is the angle between $\tilde\vw_j$ and $\tilde\vw_{j'}$.
\end{definition}
One might wonder this is hard since the student's intermediate layer receives no \emph{direct supervision} from the corresponding teacher layer, but relies only on backpropagated gradient. Surprisingly, Lemma~\ref{thm:same-layer-relationship} shows that the supervision is implicitly carried from layer to layer via gradient:

\begin{lemma}[Recursive Gradient Rule]
\label{thm:same-layer-relationship}
    At layer $l$, the backpropagated $\vg_l(\vx)$ satisfies
    \begin{equation}
        \vg_l(\vx) = D_l(\vx)\left[A_l(\vx) \vf^*_l(\vx) - B_l(\vx) \vf_l(\vx) \right] \label{eq:compatibility}
    \end{equation}
\end{lemma}

\textbf{Remark}. Lemma~\ref{thm:same-layer-relationship} applies to arbitrarily deep ReLU networks and allows $n_l \neq m_l$. In particular, student can be over-realized. Note that $A_l$, $B_l$, $V_l$, $V_l^*$ all depends on $\vx$. Due to the property of (Leaky) ReLU, $A_l(\vx)$ and $B_l(\vx)$ are piecewise constant functions (Corollary~\ref{co:piece-wise-constant} in Appendix). 

\textbf{Relationship to Network Distillation.} Our setting is closely related but different from network distillation extensively used in practice. From our point of view, both the ``pre-trained teacher'' and ``condensed student'' in network distillation are large student networks, and the dataset are samples from an inaccessible and small teacher network, i.e., the oracle. We will study connections to network distillation in our future works.

\section{Simple Example: Two-layer Network, Zero Gradient and Infinite Samples}
\label{sec:simple-example}

\begin{figure*}
    \centering
    \includegraphics[width=0.95\textwidth]{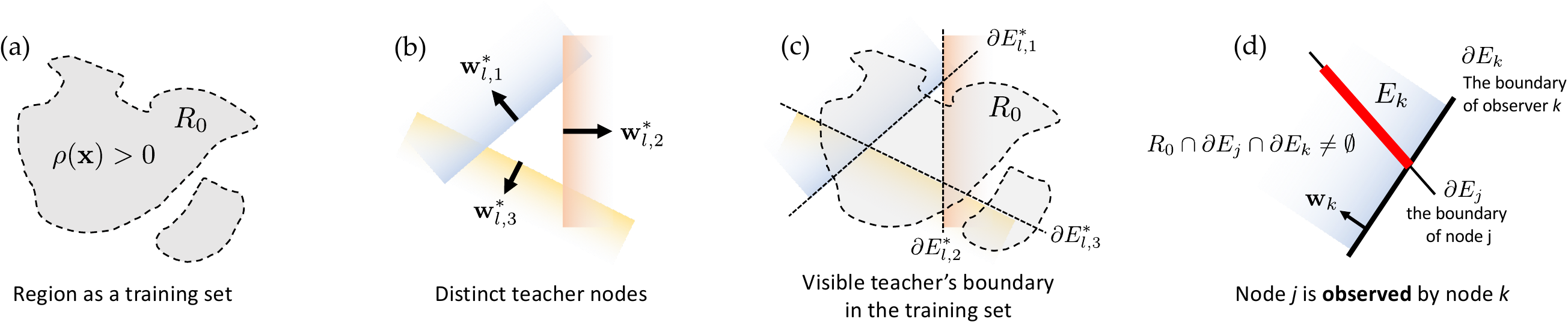}
    \vspace{-0.1in}
    \caption{\small{Simple Example (Sec.~\ref{sec:simple-example}). \textbf{(a-c)} Assumption~\ref{assumption:noise-free}: Training set is an \emph{open} region $R_0$ in the input space. Teacher nodes ($l = 1$) are distinct. Teacher boundaries are \emph{visible} in $R_0$. Here $\partial E^*_{l,j} \cap R_0 \neq \emptyset$ for $j = 1, 2, 3$. \textbf{(d)} The definition of observer (Definition~\ref{def:observer}).}}
    \vspace{-0.1in}
    \label{fig:assumptions-infinite}
\end{figure*}

We first consider a two-layer ReLU network trained with infinite samples until the gradient is zero at every training sample. This ideal case reveals key structures of student specialization with intuitive proof. In 2-layer case, $A_1(\vx)$ and $B_1(\vx)$ are constant with respect to $\vx$, since there is no ReLU gating at the top layer $l=2$. The subscript $l$ is omit for brevity. Sec.~\ref{sec:main-theorem} proposes main theorems that consider finite sample, small gradient and multi-layer networks.

Obviously, some teacher networks cannot be reconstructed, e.g., a teacher network that outputs identically zero. Therefore, some assumptions on teacher network are needed. 

\begin{assumption}
\label{assumption:noise-free}
(1) We have an infinite training set $R_0$, an \emph{open} set around the origin. (2) Any two teacher weights are not co-linear; (3) The boundary of any teacher node $j$ intersects with $R_0$: $R_0 \cap \partial E^*_j \neq \emptyset$. See Fig.~\ref{fig:assumptions-infinite}.
\end{assumption}
Intuitively, we want the teacher nodes to be well-separated, and all boundaries of teachers pass through a dataset, which reveals their nonlinear nature. ``Open set'' means $R$ has interior and is \emph{full rank}. 

\begin{definition}[Observer]
\label{def:observer}
Node $k$ is an observer of node $j$ if $R_0 \cap E_k\cap \partial E_j \neq \emptyset$. See Fig.~\ref{fig:assumptions-infinite}(d).
\end{definition}

We assume the following \emph{zero gradient condition}, which is feasible since student network is over-realized. 
\begin{assumption}[Zero Batch Gradient in SGD]
\label{assumption:zero-gradient-per-batch}
For every batch $\cB\subseteq R_0$, $\dot W_l = \eee{\vx \in \cB}{\vf_{l-1}(\vx)\vg^\t_l(\vx)} = 0$.  
\end{assumption}

Given these, we now arrive at the following theorem:

\def\vis{\mathrm{vis}}
\def\coli{\mathrm{co\mbox{-}linear}}

\begin{figure*}
    \centering
    \includegraphics[width=0.9\textwidth]{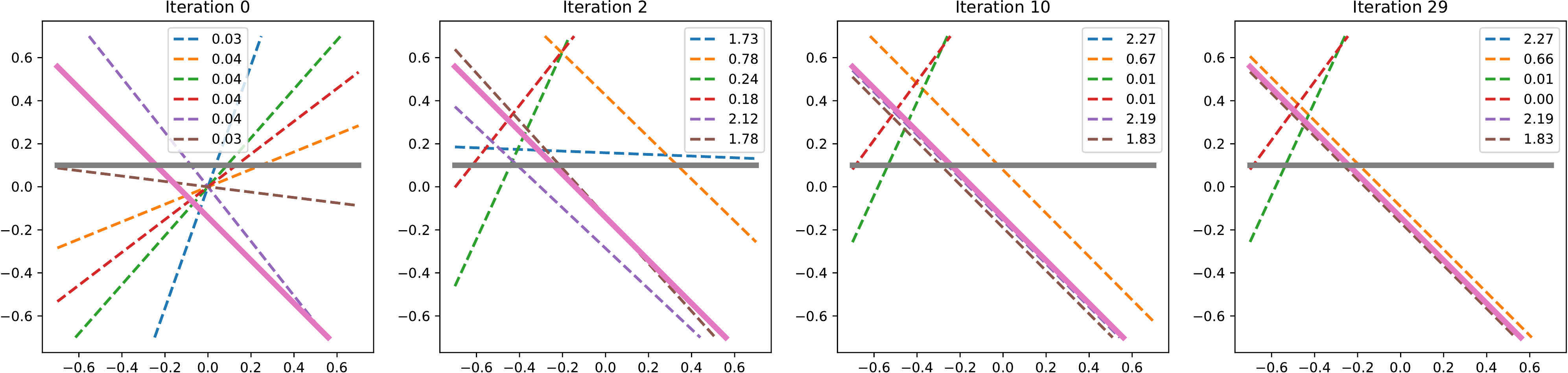}
    \vspace{-0.1in}
    \caption{\small{Convergence (2 dimension) for 2 teachers (solid line) and 6 students (dashed line). Legend shows $\|\vv_k\|$ for student node $k$. $\|\vv_k\|\rightarrow 0$ for nodes that are not aligned with teacher.}}
    \vspace{-0.1in}
    \label{fig:visualization}
\end{figure*}

\begin{theorem}[Student Specialization, 2-layers]
\label{thm:2-layer}
If Assumption~\ref{assumption:noise-free} and Assumption~\ref{assumption:zero-gradient-per-batch} hold, and \textbf{(1)} A teacher node $j$ is observed by a student node $k$, and \textbf{(2)} $\alpha_{kj} \neq 0$ (defined in Def.~\ref{def:V-alpha-beta}), then there exists one student node $k'$ with $0$-alignment (exact alignment) with $j$.
\end{theorem}

\emph{Proof sketch}. Note that ReLU activations $\sigma(\vw_k^\t\vx)$ are mutually linear independent, if their boundaries are within the training region $R_0$. On the other hand, the gradient of each student node $k$ \emph{when active}, is $\valpha^\t_k\vf_1(\vx) - \vbeta^\t_k\vf_1(\vx) = 0$, a linear combination of teacher and student nodes (note that $\valpha^\t_k$ and $\vbeta^\t_k$ are $k$-th rows of $A_1$ and $B_1$). Therefore, zero gradient means that the summation of coefficients of co-linear ReLU nodes is zero. Since teachers are not co-linear (Assumption~\ref{assumption:noise-free}), a non-zero coefficient $\alpha_{kj}\neq 0$ for teacher $j$ means that it has to be co-linear with at least one student node, so that the summation of coefficients is zero. For details, please check detailed proofs in the Appendix. 

Note that for one teacher node, multiple student nodes can specialize to it. For deep linear models, specialization does not happen since a linear subspace spanned by intermediate layer can be represented by different sets of bases. 

Note that a necessary condition of a reconstructed teacher node is that its boundary is in the active region of student, or is \emph{observed} (Definition~\ref{def:observer}). This is intuitive since a teacher node which behaves like a linear node is partly indistinguishable from a bias term. 

For student nodes that are not aligned with any teacher node, if they are observed by other student nodes, then following a similar logic, we have the following:

\begin{theorem}[Un-specialized Student Nodes are Prunable]
\label{thm:net-zero}
If an unaligned student $k$ has $C$ independent observers, i.e., the $C$-by-$C$ matrix stacking the fan-out weights $\vv$ of these observers is full rank, then $\sum_{k'\in \coli(k)} \vv_{k'}\|\vw_{k'}\| = \vzero$. If node $k$ is not co-linear with other students, then $\vv_k = \vzero$.
\end{theorem}
\begin{corollary}
\label{co:zero-contribution}
With sufficient observers, the contribution of all unspecialized student nodes is zero. 
\end{corollary}

Theorem~\ref{thm:net-zero} and Corollary~\ref{co:zero-contribution} explain why network pruning is possible~\citep{lecun1990optimal,hassibi1993optimal,hu2016network}. Note that a related theorem (Theorem 6) in~\citep{laurent2017multilinear} studies 2-layer network with scalar output and linear separable input, and discusses characteristics of individual data point contributing loss in a local minima of GD. In our paper, no linear separable condition is imposed.

\textbf{Network representations}. To compare the intermediate representation of networks trained with different initialization, previous works use Canonical Correlation Analysis (CCA)~\cite{hardoon2004canonical} and its variants (e.g., SVCCA~\cite{raghu2017svcca}) that linearly transform activations of two networks into a common aligned space. This can be explained by Theorem~\ref{thm:2-layer} and Theorem~\ref{thm:net-zero}: multiple student nodes who specialize to one teacher node can be aligned together after linear transformation and un-aligned students can be suppressed by a null transform. 

\textbf{Connectivity between two low-cost solutions.} Previous works~\cite{garipov2018loss, draxler2018essentially} discovered that low-cost solutions for neural networks can be connected via line segments, but not a single straight line. Our framework can explain this phenomenon using a construction similar to~\citep{DBLP:journals/corr/abs-1906-06247} but without the assumption of $\epsilon$-dropout stableness of a trained network using Theorem~\ref{thm:net-zero}. See Appendix~\ref{sec:connectivity} for the construction. 

\section{Main Theorems}
\label{sec:main-theorem}
In practice, stochastic gradient fluctuates around zero and only finite samples are available during training. In this case, we show a rough specialization still follows. For convenience, we define hyperplane band $I(\epsilon)$ as follows:
\begin{definition}[Hyperplane band $I_\vw(\epsilon)$]
    $I_\vw(\epsilon) = \{\vx: |\vw^\t\vx| \le \epsilon\}$. We use $I_j(\epsilon)$ if $\vw$ is from node $j$.  
\end{definition}

\begin{definition}[$(\eta, \mu)$-Dataset]
A dataset $D$ is called $(\eta, \mu)$-dataset, if there exists $\eta, \mu > 0$ so that for any regular weight $\vw$, the number of samples in the hyperplane band $N_D\left[I_\vw(\epsilon)\right] \equiv N[D \cap I_\vw(\epsilon)]$ satisfies: 
\begin{equation}
    N_D\left[I_\vw(\epsilon)\right] \le \eta \epsilon N_D + (d + 1) \label{eq:eta}
\end{equation}
and for any regular $\vw = [\tilde\vw, b]$ with $b = 0$ (no bias):
\begin{equation}
    N_D\left[D \backslash I_\vw\left(1/\epsilon\right)\right] \equiv
    N_D\left[|\tilde\vw^\t\tilde\vx| \ge \frac{1}{\epsilon}\right] \le \mu \epsilon^2 N_D \label{eq:mu}
\end{equation}
where $N_D$ is the size of the dataset.  
\end{definition}
Intuitively, Eqn.~\ref{eq:eta} means that each point of the dataset is scattered around and any hyperplane band $|\vw^\t\vx| \le \epsilon$ cannot cover them all. It is in some sense a high-rank condition for the dataset. The additional term $d+1$ exists because there always exists a hyperplane $\vw_0$ that passes any $d+1$ points (excluding degenerating case). Therefore, $N_D[\vw_0^\t\vx = 0] = d+1$. Eqn.~\ref{eq:mu} can be satisfied with dataset sampled by any zero-mean distribution with finite variance, due to Chebyshev's inequality: $\pr[|\tilde\vw^\t\tilde\vx| \ge 1 / \epsilon] \le \epsilon^2 \mathrm{Var}(\tilde\vw^\t\tilde\vx)$. 

\begin{assumption}
\label{assumption:main}
(a) Two teacher nodes $j\neq j'$ are not $\epsilon_0$-aligned. (b) The boundary band $I_j(\epsilon)$ of each teacher $j$ overlaps with the dataset: 
\begin{equation}
N_D[I_j(\epsilon)] \ge \tau \epsilon N_D \label{eq:teacher-finite-assumption}
\end{equation}
\end{assumption}
\vspace{-0.05in}
Intuitively, this means that we have sufficient samples near the boundary of each teacher nodes, in order to correctly identify the weights $\vw^*_j$ of each teacher node $j$. Again, a teacher node that are not visible from the training set cannot be reconstructed. For this purpose, the reader might wonder why Eqn.~\ref{eq:teacher-finite-assumption} does not impose constraints that there need to be input samples on both sides of the teacher boundary $\partial E^*_j$. The answer is that we also assume proper data augmentation (see Definition~\ref{def:augmentation}). 

\def\aug{\mathrm{aug}}

\begin{figure}
    \centering
    \includegraphics[width=0.5\textwidth]{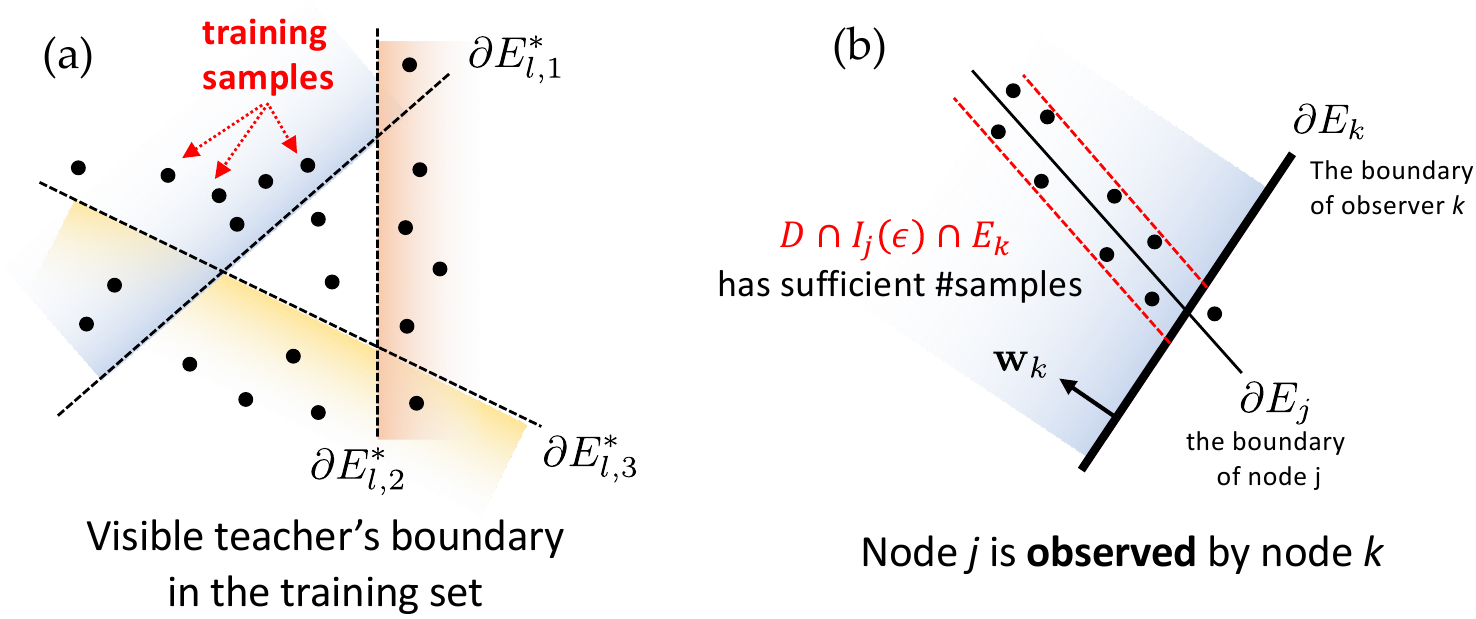}
    \caption{\small{Settings in Main Theorems (Sec.~\ref{sec:main-theorem}). \textbf{(a)} Assumption~\ref{assumption:main}: Teacher boundaries are \emph{visible} in dataset $D$. \textbf{(b)} The definition of observer now incorporates sample counts.}}
    \label{fig:assumptions-finite}
\end{figure}

\subsection{Two Layer Case}
The following two theorems show that with small gradient and polynomial number of samples, a rough specialization still follows in 2-layer network. Here let $K_1 = m_1 + n_1$.

\begin{definition}[Augmentation]
\label{def:augmentation}
Given a dataset $D$, we construct $\aug(D)$ as \textbf{(a)} Teacher-agnostic:
\[
    \aug(D) = \{\vx \pm 2\epsilon\tilde\ve_u/cK^{3/2}_1,\ \ \vx \in D,\ \ u = 1,\ldots,d\} \cup D
\]
where $\tilde\ve_k$ is axis-aligned unit directions with $\|\tilde\ve_k\| = 1$ or \textbf{(b)} Teacher-aware: 
\[
    \aug(D, \cW^*) = \{\vx \pm 2\epsilon\tilde\vw^*_j/cK^{3/2}_1,\ \ \vx \in D \cap I_{\vw^*_j}(\epsilon)\} \cup D
\]
In both definitions, $c$ is a constant related to $(\eta,\mu)$ of $D$ (See proof of Theorem~\ref{thm:finite-sample-recovery-appendix} in Appendix). 
\end{definition}

With the data augmentation, a polynomial number of samples suffice for student specialization to happen. 

\begin{theorem}[Two-layer Specialization with Polynomial Samples]
\label{thm:finite-sample-recovery}
For $0<\epsilon\le\epsilon_0$ and $0<\kappa\le 1$, for any finite dataset $D$ with $N = \cO(K_1^{5/2}d^2\epsilon^{-1}\kappa^{-1})$, for any teacher satisfying Assumption~\ref{assumption:main} and student trained on $D' = \aug(D)$ whose weight $\hW$ satisfies:
\begin{itemize}
    \item[(1)] For $0 < \epsilon < \epsilon_0$, $I_j(\epsilon)$ of a teacher $j$ is \emph{observed} by a student node $k$: $N_D[I_j(\epsilon) \cap E_k] \ge \kappa N_D[I_j(\epsilon)]$;
    \item[(2)] Small gradient: $\|\vg_1(\vx, \hW)\|_\infty \le \frac{\alpha_{kj}}{5K^{3/2}_1\sqrt{d}} \epsilon,\,\vx\in D'$,
\end{itemize}
then there exists a student $k'$ so that $(j, k')$ is $\epsilon$-aligned.
\end{theorem}

\begin{theorem}[Two-layer Specialization with Teacher-aware Dataset with Polynomial Samples] 
\label{thm:finite-sample-teacher-aware-recovery}
For $0 < \epsilon \le \epsilon_0$ and $0<\kappa\le 1$, for any finite dataset $D$ with $N = \cO(K^{5/2}_1d\epsilon^{-1}\kappa^{-1})$, given a teacher network $\cW^*$ satisfying Assumption~\ref{assumption:main} and student trained on $D' = \aug(D, \cW^*)$ whose weight $\hW$ satisfies:
\begin{itemize}
    \item[(1)] For $0 < \epsilon < \epsilon_0$, $I_j(\epsilon)$ of a teacher $j$ is \emph{observed} by a student node $k$: $N_D[I_j(\epsilon) \cap E_k] \ge \kappa N_D[I_j(\epsilon)]$;
    \item[(2)] Small gradient: $\|\vg_1(\vx, \hW)\|_\infty \le \frac{\alpha_{kj}}{5K^{3/2}_1} \epsilon$, for $\vx \in D'$,
\end{itemize}
then there exists a student $k'$ so that $(j, k')$ is $\epsilon$-aligned.
\end{theorem}
\textbf{Remark}. Note that the definition of observation changes slightly in the finite sample setting: a sufficient number of samples needs to be observed in order to show convergence. Note that Theorem~\ref{thm:finite-sample-recovery} contains an additional $\sqrt{d}$ factor in the gradient condition, due to teacher-agnostic augmentation. In fact, following the same proof idea, to remove the factor $\sqrt{d}$, exponential number of samples are needed, or knowledge of the teacher network (Theorem~\ref{thm:finite-sample-teacher-aware-recovery}). This also suggests that for any teacher network $\vf^*(\vx)$, there exists compatible input distribution so that the dataset $\{(\vx, \vf^*(\vx))\}_{\vx\in D}$ can be easy to learn. 

\subsection{Multi-layer Case}
As in the 2-layer case, we can use similar intuition to analyze the behavior of the lowest layer for deep ReLU networks, thanks to Lemma~\ref{thm:same-layer-relationship} which holds for arbitrarily deep ReLU networks. In this case, $A_1(\vx)$ and $B_1(\vx)$ are no longer constant over $\vx$, but are piece-wise constant. Note that $A_1(\vx)$ and $B_1(\vx)$ might contain exponential number of regions, $\cR = \{R_0, R_1, \ldots, R_J\}$, where in each region $R$, $A_l(\vx)$ and $B_l(\vx)$ are constants. 

As intersection of regions, the number boundaries are also exponential. The underlying intuition is that for each intermediate node, its boundary is ``bended'' to another direction, whenever the boundary meets with any boundary of its input nodes~\cite{hanin2019deep, rolnick2019identifying, hanin2019complexity}. All these boundaries will be reflected in the input region. Therefore, the number $Q$ of hyper plane boundaries is exponential with respect to $L$, leading to exponential sample complexity.

\begin{theorem}[Multi-layer alignment, Lowest Layer]
\label{thm:multi-layer}
Given $0 < \epsilon \le \epsilon_0$, for any finite dataset $D$ with $N = \cO(Q^{5/2}d^2\epsilon^{-1}\kappa^{-1})$, if the first layer of the deep teacher network satisfies Assumption~\ref{assumption:main} and any student weight at the first layer $\hW_1$ satisfies:
\begin{itemize}
    \item[(1)] For $0 < \epsilon < \epsilon_0$, $I_j(\epsilon)$ of a teacher $j$ is observed by a student node $k$: $N_D[I_j(\epsilon) \cap E_k] \ge \kappa N_D[I_j(\epsilon)]$;
    \item[(2)] $\|\vg_1(\vx, \hW)\|_\infty \le \frac{\min_{R\in\cR} \alpha_{kj}(R)}{5Q^{3/2}\sqrt{d}}\epsilon$, for $\vx \in D'$,
\end{itemize}
then there exists student node $k'$ so that $(j, k')$ is $\epsilon$-aligned. 
\end{theorem}

\vspace{-0.05in}
\section{Discussions}
\vspace{-0.05in}
The theorems suggest a few interesting consequences:

\textbf{Strong and weak teacher nodes}. Theorems show that the convergence is dependent on $\alpha_{kj} = \vv_k^\t\vv^*_j$, where $\vv^*_j$ is the $j$-th column of $V_1^*$ and $\vv_k$ is the $k$-th column of $V_1^*$. Given the same magnitude of gradient norm, \emph{strong} teacher $j$ (large $\|\vv_j^*\|$ and thus large $\alpha_{kj}$) would achieve specialization, while weak teacher will not achieve specialization. This explains why early stopping~\cite{caruana2001overfitting} could help and suggests how the inductive bias is created during training. We verify this behavior in our experiments.   

\textbf{Dataset matters}. Theorem~\ref{thm:finite-sample-recovery} and theorem~\ref{thm:finite-sample-teacher-aware-recovery} shows different sample complexity for datasets that are augmented with different augmentation methods, one is teacher-agnostic while the other is teacher-aware. This shows that if dataset is \emph{compatible} with the underlying teacher network, the specialization would be much faster. We also verify this behavior in our experiments, showing that training with teacher-aware dataset yields much faster convergence as well as stronger generalization, given very few number of samples. 

\textbf{The role of over-realization}. Theorem~\ref{thm:2-layer} suggests that over-realization (more student nodes in the hidden layer $l=1$) is important. More student nodes mean more observers, and the existence argument in these theorems is more likely to happen and more teacher nodes can be covered by student, yielding better generalization. 

\textbf{Expected SGD conditions}. In practice, the gradient conditions might hold in expectation (or high probability), e.g. $\eee{t}{\|\vg_1(\vx)\|_\infty} \le \epsilon$. This means that $\|\vg_1(\vx)\|_\infty \le \epsilon$ at least for some iteration $t$. All theorems still apply since they do not rely on past history of the weight or gradient. 

\def\ncor{\rho_{\mathrm{mean}}}

\begin{figure*}
    \centering
    \includegraphics[width=0.8\textwidth]{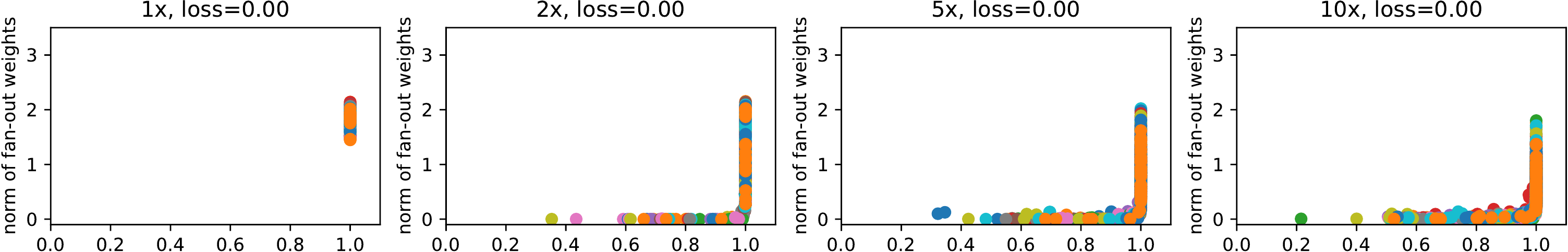}
    \vspace{-0.1in}
    \caption{\small{Student specialization of a 2-layered network with 10 teacher nodes and $1\times$/$2\times$/$5\times$/$10\times$ student nodes.  For a student node $k$, we plot its degree of specialization (i.e., normalized correlation to its best correlated teacher) as the $x$ coordinate and the fan-out weight norm $\|\vv_k\|$ as the $y$ coordinate. We plot results from 32 random seed. Student nodes of different seeds are in different color. An un-specialized student node has low correlations with teachers and low fan-out weight norm (Theorem~\ref{thm:net-zero}). Higher $p$ makes reconstruction of teacher node harder, in particular if the student network is not over-realized.}}.
    \vspace{-0.1in}
    \label{fig:max-corr-at-convergence}
\end{figure*}

\section{Experiments}
We first verify our theoretical finding on synthetic dataset generated by Gaussian distribution $\cN(0, \sigma^2 I)$ with $\sigma = 10$. With other distributions (e.g., uniform $U[-1, 1]$), the result is similar. Appendix ~\ref{sec:teacher-construction} gives details on how we construct the teacher network and training/eval dataset. Vanilla SGD is used with learning rate $0.01$ and batchsize $16$. To measure the degree of student specialization, we define $\ncor$ as the mean of maximal normalized correlation ($\tilde\vf_{j}$ is the normalized activation of node $j$ over the evaluation set):
\begin{equation}
    \ncor = \underset{j\in\ \mathrm{teacher}}{ \mathrm{mean}} \max_{j'\in\ \mathrm{student}} \rho_{jj'}, \quad \rho_{jj'} = \tilde\vf^{*\t}_j\tilde\vf_{j'}, 
\end{equation}
\textbf{Strong/weak teacher nodes}. To demonstrate the effect of strong and weak teacher nodes, we set up a diverse strength of teacher node by constructing the fan-out weights of teacher node $j$ as follows:
\begin{equation}
    \|\vv^*_j\| \sim 1 / j^p, \label{eq:teacher-fanout-decay}
\end{equation}
where $p$ is the \emph{teacher polarity factor} that controls how strong the energy decays across different teacher nodes. $p = 0$ means all teacher nodes have the same magnitude of fan-out weights, and large $p$ means that the strength of teacher nodes are more polarized, i.e., some teacher nodes are very strong, some are very weak. 

\textbf{Two layer networks}. First we verify Theorem~\ref{thm:2-layer} and Theorem~\ref{thm:net-zero} in the 2-layer setting. Fig.~\ref{fig:max-corr-at-convergence} shows for different degrees of over-realization ($1\times$/$2\times$/$5\times$/$10\times$), for nodes with weak specialization (i.e., its normalized correlation to the most correlated teacher is low, left side of the figure), their magnitudes of fan-out weights ($y$-axis) are small. The nodes with strong specialization have high fan-out weights. 

\textbf{Deep Networks}. 
For deep ReLU networks (4-layer), we observe student specialization at \emph{each} layer, shown in Fig.~\ref{fig:sample-complexity-deep}. We can also see the lowest layer converges better than the top layers at different sample sizes, in particular with MSE loss. Although our theory does not apply to Cross Entropy loss yet, empirically we still see specialization in multiple layers, in particular at the lowest layer.    

\subsection{Ablation studies}
\textbf{Strong/weak teacher node}. We plot the average rate of a teacher node that is matched with at least one student node successfully (i.e., correlation $>0.95$). Fig.~\ref{fig:recover-teacher-node} shows that stronger teacher nodes are more likely to be matched, while weaker ones may not be explained well, in particular when the strength of the teacher nodes are polarized ($p$ is large). This is consistent with Theorem~\ref{thm:finite-sample-recovery} which shows that a strong teacher node can be specialized even if the gradient magnitude is still relatively large, compared to a weak one. Fig.~\ref{fig:best-corr-over-iteration} further shows the dynamics of specialization for each teacher node: strong teacher node gets specialized fast, while it takes very long time to have students specialized to weak teacher nodes.

\textbf{Effects of Over-realization}. Over-realized student can explain more teacher nodes, while a student with $1\times$ nodes has sufficient capacity to fit the teacher perfectly, it gets stuck despite long training. 

\textbf{Sample Complexity}. Fig.~\ref{fig:sample-complexity-deep} shows hows node correlation (and generalization) changes when sample complexity when we use MSE or Cross Entropy Loss. With more samples, generalization becomes better and $\ncor$ also becomes better. Note that although our analysis doesn't include CE Loss, empirically we see $\ncor$ grows, in particular at the lowest layer, when $N$ becomes large. 

\textbf{Teacher-agnostic versus Teacher-aware}. Theorem~\ref{thm:finite-sample-teacher-aware-recovery} and Theorem~\ref{thm:finite-sample-recovery} shows that with different datasets (or same dataset but with different data augmentation), sample complexity can be very different. As shown in Fig.~\ref{fig:aware-agnostic}, if we construct a dataset with the knowledge of teacher, then a student trained on it will not overfit even with small number of samples. This shows the dependency of sample complexity with respect to the dataset. 

\begin{figure*}
    \centering
    \includegraphics[width=\textwidth]{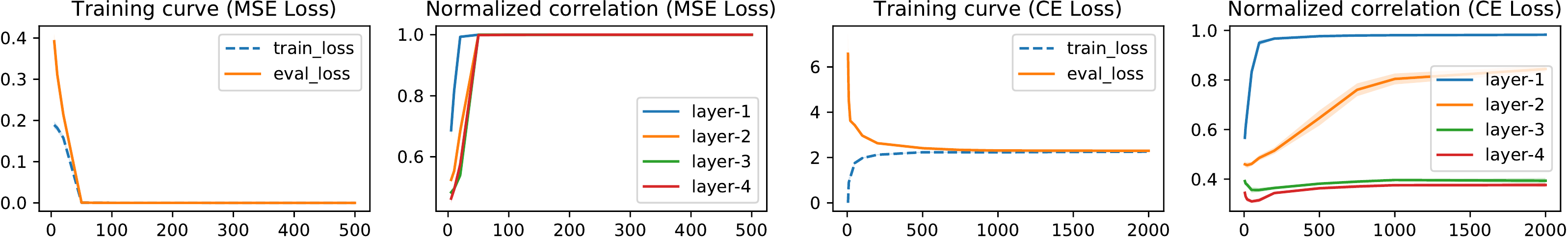}
    \caption{\small{Relationship between evaluation loss and normalized correlation $\ncor$  ($y$-axis), and sample complexity ($x$-axis, $\times 1000$) for MSE/Cross Entropy (CE) Loss function. Teacher is 4-layer with 50-75-100-125 hidden nodes and student is $2\times$ over-realization.}}
    \label{fig:sample-complexity-deep}
\end{figure*}

\begin{figure*}
    \centering
    \includegraphics[width=\textwidth]{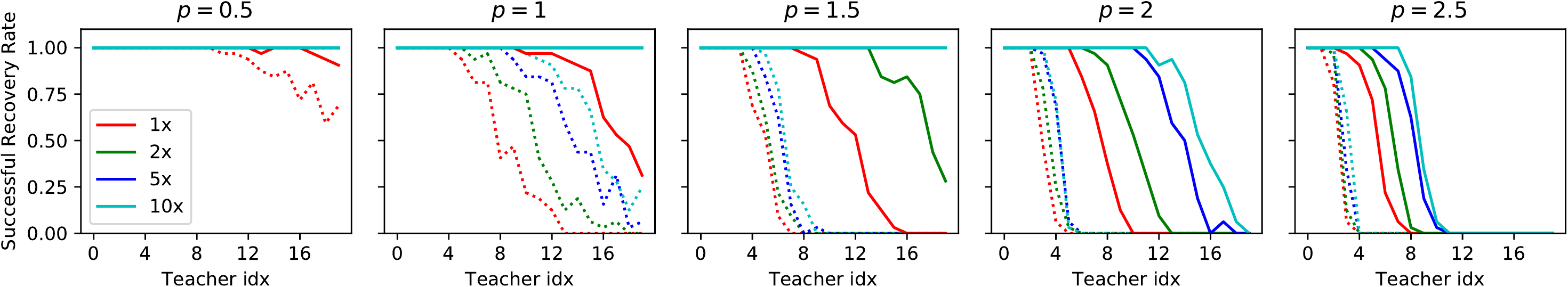}
    \caption{\small{Success rate (over 32 random seeds) of recovery of 20 teacher nodes on 2-layer network at different teacher polarity $p$ (Eqn.~\ref{eq:teacher-fanout-decay}) under different over-realization. Dotted line: successful rate after 5 epochs. Solid line: successful rate after 100 epochs. }}
    \label{fig:recover-teacher-node}
\end{figure*}

\begin{figure*}
    \centering
    \includegraphics[width=\textwidth]{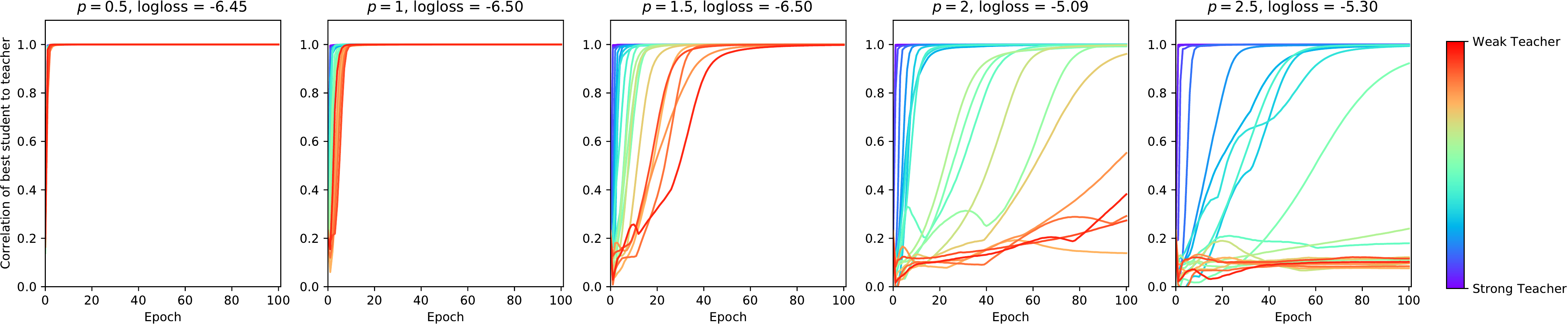}
    \caption{\small{Dynamics of specialization (2-layer) over iterations. Each rainbow color represents one of the 20 teacher nodes (blue: strongest teacher, red: weakest). Strong teacher nodes (blue) can get specialized very quickly, while weak teacher nodes (red) is not specialized for a long time. The student is $5\times$ over-realization. Large $p$ means strong polarity of teacher nodes (Eqn.~\ref{eq:teacher-fanout-decay}).}}
    \label{fig:best-corr-over-iteration}
\end{figure*} 

\begin{figure*}
    \centering
    \includegraphics[width=\textwidth]{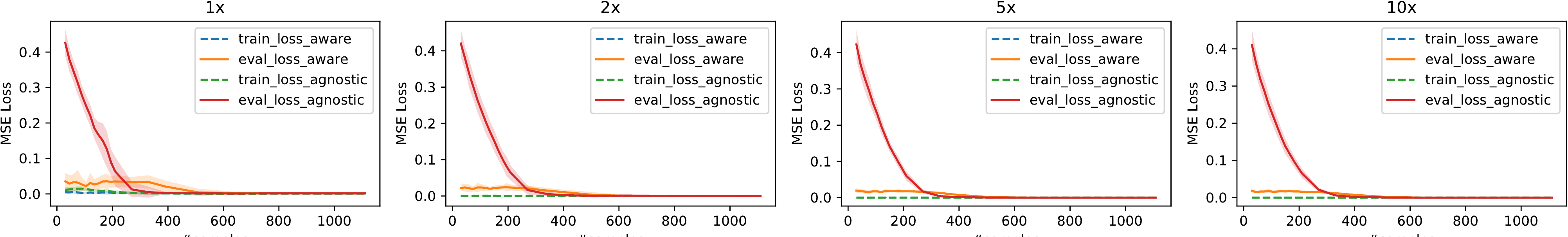}
    \caption{\small{MSE Loss ($y$-axis) versus number of samples used ($x$-axis) using teacher-aware and teacher agnostic dataset in 2-layer network ($10$ random seeds). While training loss is low on both cases, teacher-aware dataset leads to substantial lower evaluation loss.}}
    \label{fig:aware-agnostic}
\end{figure*}

\begin{figure*}
    \centering
    \includegraphics[width=0.9\textwidth]{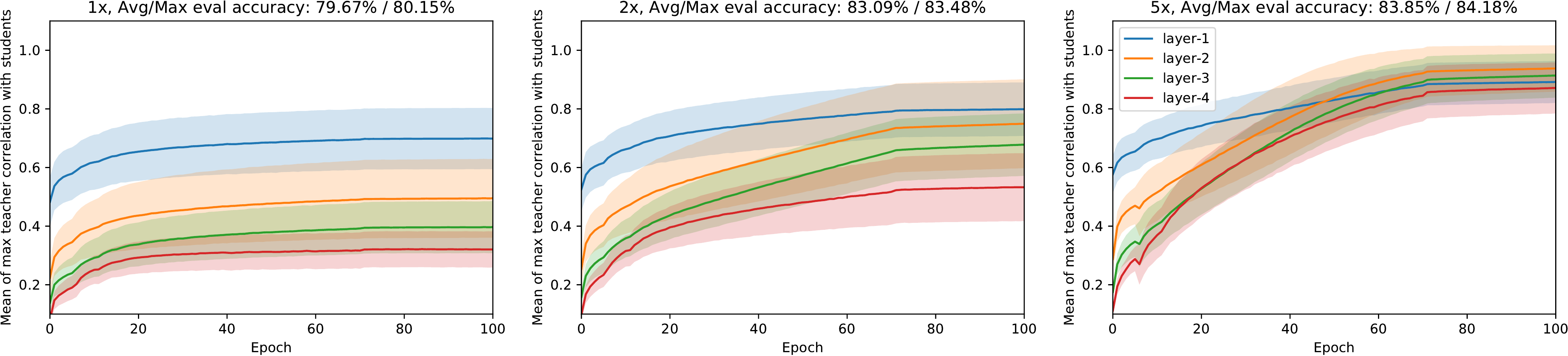}
    \caption{\small{Mean of the max teacher correlation $\rho_{\mathrm{mean}}$ with student nodes over epochs in CIFAR10. More over-realization gives better student specialization across all layers and achieves strong generalization (higher evaluation accuracy on CIFAR-10 evaluation set).}}
    \label{fig:cifar-10}
\end{figure*}

\textbf{CIFAR-10}. We also experiment on CIFAR-10. We first pre-train a teacher network with 64-64-64-64 ConvNet ($64$ are channel sizes of the hidden layers, $L = 5$) on  CIFAR-10 training set. Then the teacher network is pruned in a structured manner to keep strong teacher nodes. The student is over-realized based on teacher's remaining channels.

Fig.~\ref{fig:cifar-10} shows the convergence and specialization behaviors of student network. More over-realization leads to stronger specialization and improved generalization evaluated on CIFAR-10 evaluation set. 

\section{Conclusion and Future Work}
\vspace{-0.1in}
In this paper, we use student-teacher setting to analyze how an (over-realized) deep ReLU student network trained with SGD learns from the output of a teacher. When the magnitude of gradient per sample is small, the teacher can be proven to be specialized by (possibly multiple) students and thus the teacher network is recovered at the lowest layer. We also provide finite sample analysis about when it happens. As future works, it is interesting to show specialization at \emph{every} layer in deep networks and understand training dynamics in the teacher-student setting. On the empirical side, student-teacher setting can be useful when analyzing many practical phenomena (e.g., connectivity of low-cost solutions). 

\section*{Acknowledgement}
The author thanks Banghua Zhu, Xiaoxia (Shirley) Wu, Lexing Ying, Jonathan Frankle, Ari Morcos, Simon Du, Sanjeev Arora, Wei Hu, Xiang Wang, Yang Yuan, Song Zuo and anonymous reviewers for insightful comments. 

\bibliographystyle{icml2020}
\bibliography{main3}

\newif\ifsupp
\supptrue

\ifsupp
\clearpage

\onecolumn
\appendix

\icmltitle{Supplementary Materials for Student Specialization in Deep Rectified Networks With Finite Width and Input Dimension}

\section{More related works}
\textbf{Deep Linear networks}. For deep linear networks, multiple works~\citep{lampinen2018analytic,saxe2013exact,arora2018a,advani2017high} have shown interesting training dynamics. One common assumption is that the singular spaces of weights at nearby layers are aligned at initialization, which decouples the training dynamics. Such a nice property would not hold for nonlinear network. ~\citep{lampinen2018analytic} shows that under this assumption, weight components with large singular value are learned first, while we analyze and observe empirically similar behaviors on the student node level. Generalization property of linear networks can also be analyzed in the limit of infinite input dimension with teacher-student setting~\citep{lampinen2018analytic}. However, deep linear networks lack specialization which plays a crucial role in the nonlinear case. To our knowledge, we are the first to analyze specialization rigorously in deep ReLU networks. 

\textbf{SGD versus GD}. Stochastic Gradient Descent (SGD) shows strong empirical performance than Gradient Descent (GD)~\citep{shallue2018measuring} in training deep models. SGD is often treated as an approximate, or a noisy version of GD~\citep{bertsekas2000gradient,hazan2014beyond,marceau2017natural,goldt2019dynamics,bottou2010large}. In contrast, many empirical evidences show that SGD achieves better generalization than GD when training neural networks, which is explained via implicit regularization~\citep{rethinking,neyshabur2015path}, by converging to flat minima~\citep{hochreiter1997flat,chaudhari2016entropy,wu2018sgd} , robust to saddle point~\citep{jin2017escape,daneshmand2018escaping,ge2015escaping,exponentialtime} and perform Bayesian inference~\citep{welling2011bayesian, mandt2017stochastic, chaudhari2018stochastic}. 

Similar to this work, interpolation setting~\citep{ma2017power,liu2018mass,bassily2018exponential} assumes that gradient at each data point vanish at the critical point. While they mainly focus on convergence property of convex objective, we directly relate this condition to specific structure of deep ReLU networks.

\section{A Mathematical Framework}
\subsection{Lemma~\ref{thm:same-layer-relationship}}
\begin{proof}
We prove by induction. When $l = L$ we know that $\vg_L(\vx) = \vf^*_L(\vx) - \vf_L(\vx)$, by setting $V^*_L(\vx) = V_L(\vx) = I_{C\times C}$ and the fact that $D_L(\vx) = I_{C\times C}$ (no ReLU gating in the last layer), the condition holds. 

Now suppose for layer $l$, we have:
\begin{eqnarray}
\vg_l(\vx) &=& D_l(\vx)\left[A_l(\vx) \vf^*_l(\vx) - B_l(\vx) \vf_l(\vx) \right] \\
&=& D_l(\vx)V_l^\t(\vx)\left[V^*_l(\vx)\vf^*_l(\vx) - V_l(\vx)\vf_l(\vx)\right]
\end{eqnarray}

Using
\begin{eqnarray}
    \vf_l(\vx) &=& D_l(\vx)W_l^\t\vf_{l-1}(\vx) \\ \vf^*_l(\vx) &=& D^*_l(\vx)W_l^{*\t}\vf^*_{l-1}(\vx) \\ \vg_{l-1}(\vx) &=& D_{l-1}(\vx)W_l\vg_l(\vx)
\end{eqnarray}
we have:

\begin{eqnarray}
\vg_{l-1}(\vx) &=& D_{l-1}(\vx)W_l\vg_l(\vx) \\
&=& D_{l-1}(\vx)\underbrace{W_lD_l(\vx)V_l^\t(\vx)}_{V_{l-1}^\t(\vx)}\left[V^*_l(\vx) \vf^*_l(\vx) - V_l(\vx) \vf_l(\vx) \right] \\
&=& D_{l-1}(\vx)V_{l-1}^\t(\vx)\left[\underbrace{V_l^*(\vx)D_l^*(\vx)W_l^{*\t}}_{V^*_{l-1}(\vx)}\vf^*_{l-1}(\vx) - \underbrace{V_l(\vx)D_l(\vx)W_l^{\t}}_{V_{l-1}(\vx)}\vf_{l-1}(\vx)\right] \\
&=& D_{l-1}(\vx)V_{l-1}^\t(\vx)\left[V^*_{l-1}(\vx)\vf^*_{l-1}(\vx) - V_{l-1}(\vx)\vf_{l-1}(\vx)\right] \\
&=& D_{l-1}(\vx)\left[A_{l-1}(\vx) \vf^*_{l-1}(\vx) - B_{l-1}(\vx) \vf_{l-1}(\vx)\right]
\end{eqnarray}
\end{proof}

\subsection{Lemma~\ref{thm:sgd-critical-point}}
\begin{lemma}
\label{thm:sgd-critical-point}
Denote $\cD = \{\vx_i\}$ as a dataset of $N$ samples. If Assumption~\ref{assumption:zero-gradient-per-batch} holds, then either $\vg_l(\vx_i; \hW) = \vzero$ or $\vf_{l-1}(\vx_i; \hW) = \vzero$.
\end{lemma}
\begin{proof}
From Assumption~\ref{assumption:zero-gradient-per-batch}, we know that for any batch $\cB_j$, Eqn.~\ref{eq:gradient-rule} vanishes: 
\begin{equation}
    \dot W_l = \eee{\vx}{\vg_l(\vx; \hW)\vf^\t_{l-1}(\vx; \hW)} = \sum_{i\in\cB_j} \vg_l(\vx_i; \hW)\vf^\t_{l-1}(\vx_i; \hW) = 0 \label{eq:batch-zero}
\end{equation}
Let $U_i = \vg_l(\vx_i; \hW)\vf^\t_{l-1}(\vx_i; \hW)$. Note that $\cB_j$ can be any subset of samples from the data distribution. Therefore, for a dataset of size $N$, Eqn.~\ref{eq:batch-zero} holds for all $\binom{N}{|\cB|}$ batches, but there are only $N$ data samples. With simple Gaussian elimination we know that for any $i_1 \neq i_2$, $U_{i_1} = U_{i_2} = U$. Plug that into Eqn.~\ref{eq:batch-zero} we know $U = 0$ and thus for any $i$, $U_i = 0$. Since $U_i$ is an outer product, the theorem follows. 

Note that if $\|\dot W_l\|_\infty \le \epsilon$, which is $\|\sum_{i\in\cB_j} U_i\|_\infty \le \epsilon$, then with simple Gaussian elimination for two batches $\cB_1$ and $\cB_2$ with only two sample difference, we will have for any $i_1\neq i_2$, $\|U_{i_1} - U_{i_2}\|_\infty = \|\sum_{i\in\cB_1} U_i - \sum_{i\in\cB_2} U_i\|_\infty \le \|\sum_{i\in\cB_1} U_i\|_\infty + \|\sum_{i\in\cB_2} U_i\|_\infty = 2\epsilon$. Plug things back in and we have $|\cB| \|U_i\|_\infty \le [2(|\cB| - 1) + 1]\epsilon$, which is $\|U_i\|_\infty \le 2\epsilon$. If $\vf_{l-1}(\vx; \hW)$ has the bias term, then immediately we have $\|\vg_l(\vx;\hW)\|_\infty \le \epsilon$. 
\end{proof}

\section{Two-layer, Infinite Samples and Zero Gradient}
\begin{definition}[Regular weight vector]
    A weight vector $\vw = [\tilde\vw, b] \in \rr^{d+1}$ is called \emph{regular}, if $\|\tilde\vw\|_2 = 1$.
\end{definition}

\begin{figure}
    \centering
    \includegraphics[width=0.9\textwidth]{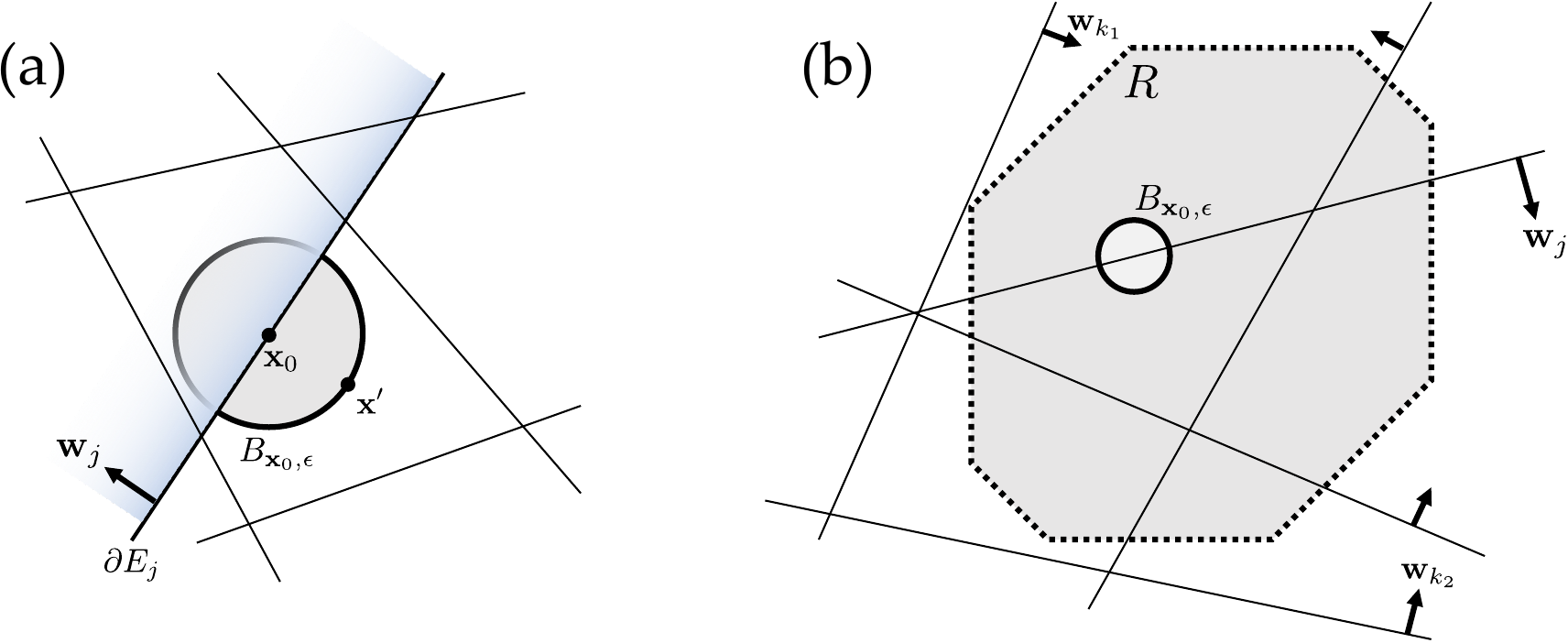}
    \caption{Proof illustration for \textbf{(a)} Lemma~\ref{lemma:colinear}, \textbf{(b)} Lemma~\ref{lemma:local-colinear}.} 
    \label{fig:proof1}
\end{figure}

\subsection{Corollary~\ref{co:piece-wise-constant}}
\begin{corollary}[Piecewise constant]
\label{co:piece-wise-constant}
$R_0$ can be decomposed into a finite (but potentially exponential) set of regions $\cR_{l-1} = \{R_{l-1}^1, R_{l-1}^2, \ldots, R_{l-1}^J\}$ plus a zero-measure set, so that $A_l(\vx)$ and $B_l(\vx)$ are constant within each region $R_{l-1}^j$ with respect to $\vx$.  
\end{corollary}
\begin{proof}
The base case is that $V_L(\vx) = V^*_L(\vx) = I_{C\times C}$, which is constant (and thus piece-wise constant) over the entire input space. If for layer $l$, $V_l(\vx)$ and $V^*_l(\vx)$ are piece-wise constant, then by Eqn.~\ref{eq:v-recursive} (rewrite it here):
\begin{equation}
    V_{l-1}(\vx) = V_l(\vx)D_l(\vx)W^\t_l, \quad V^*_{l-1}(\vx) = V^*_l(\vx)D^*_l(\vx)W^{*\t}_l
\end{equation}
since $D_l(\vx)$ and $D^*_l(\vx)$ are piece-wise constant and $W^\t_l$ and $W^{*\t}_l$ are constant, we know that for layer $l - 1$, $V_{l-1}(\vx)$ and $V^*_{l-1}(\vx)$ are piece-wise constant. Therefore, for all $l = 1, \ldots L$, $V_l(\vx)$ and $V^*_l(\vx)$ are piece-wise constant. 

Therefore, $A_l(\vx)$ and $B_l(\vx)$ are piece-wise constant with respect to input $\vx$. They separate the region $R_0$ into constant regions with boundary points in a zero-measured set.  
\end{proof}

\subsection{Lemma~\ref{lemma:colinear}}
\begin{lemma}
\label{lemma:colinear}
Consider $K$ ReLU activation functions $f_j(\vx) = \sigma(\vw_j^\t\vx)$ for $j = 1\ldots K$. If $\vw_j \neq 0$ and no two weights are co-linear, then $\sum_{j'} c_{j'}f_{j'}(\vx) = 0$ for all $\vx\in\rr^{d+1}$ suggests that all $c_j = 0$.
\end{lemma}

\begin{proof}
Suppose there exists some $c_j \neq 0$ so that $\sum_j c_jf_j(\vx) = 0$ for all $\vx$. Pick a point $\vx_0 \in \partial E_j$ so that $\vw_j^\t\vx_0 = 0$ but all $\vw_{j'}^\t\vx_0 \neq 0$ for $j' \neq j$, which is possible due to the distinct weight conditions. Consider an $\epsilon$-ball $B_{\vx_0, \epsilon} = \{\vx : \|\vx-\vx_0\| \le \epsilon\}$. We pick $\epsilon$ so that $\mathrm{sign}(\vw_{j'}^\t\vx)$ for all $j'\neq j$ remains the same within $B_{\vx_0, \epsilon}$ (Fig.~\ref{fig:proof1}(a)). Denote $[j^+]$ as the indices of activated ReLU functions in $B_{\vx_0, \epsilon}$ except $j$. 

Then for all $\vx \in B_{\vx_0, \epsilon} \cap E_j$, we have:
\begin{equation}
    h(\vx) \equiv \sum_{j'} c_{j'} f_{j'}(\vx) = c_j \vw_j^\t\vx + \sum_{j'\in[j^+]} c_{j'}\vw_{j'}^\t\vx = 0 \label{eq:hzero}
\end{equation}
Since $B_{\vx_0, \epsilon}$ is a $d$-dimensional object rather than a subspace, for $\vx_0$ and $\vx_0 + \epsilon\ve_k \in B(\vx_0, \epsilon)$, we have 
\begin{equation}
    h(\vx_0 + \epsilon\ve_k) - h(\vx_0) = \epsilon(c_j w_{jk} + \sum_{j'\in[j^+]} c_{j'}w_{j'k}) = 0    
\end{equation}
where $\ve_k$ is axis-aligned unit vector ($1\le k \le d$). This yields
\begin{equation}
    c_j\tilde\vw_j + \sum_{j'\in [j^+]} c_{j'}\tilde\vw_{j'} = \vzero_d
\end{equation}
Plug it back to Eqn.~\ref{eq:hzero} yields 
\begin{equation}
    c_j b_j + \sum_{j'\in[j^+]} c_{j'} b_{j'} = 0    
\end{equation}
where means that for the (augmented) $d+1$ dimensional weight:
\begin{equation}
    c_j\vw_j + \sum_{j'\in [j^+]} c_{j'}\vw_{j'} = \vzero_{d+1}
\end{equation}
However, if we pick $\vx' = \vx_0 - \epsilon \frac{\tilde\vw_j}{\|\tilde\vw_j\|^2} \in B_{\vx_0, \epsilon} \cap E_j^\complement$, then $f_j(\vx') = 0$ but $\sum_{j'\in[j^+]} f_j'(\vx') = -c_j \vw_j^\t\vx' = \epsilon c_j$ and thus 
\begin{equation}
    \sum_{j'} c_{j'} f_{j'}(\vx') = \epsilon c_j \neq 0
\end{equation}
which is a contradiction.
\end{proof}

\newcommand\sbullet[1][.5]{\mathbin{\vcenter{\hbox{\scalebox{#1}{$\bullet$}}}}}

\subsection{Lemma~\ref{lemma:local-colinear}} \begin{lemma}[Local ReLU Independence]
\label{lemma:local-colinear}
Let $R$ be an open set. Consider $K$ ReLU nodes $f_j(\vx) = \sigma(\vw_j^\t\vx)$, $j = 1, \ldots, K$. $\vw_j\neq 0$, $\vw_j \neq \gamma\vw_{j'}$ for $j\neq j'$ with any $\gamma > 0$. 

If there exists $c_1, \ldots, c_K, c_{\sbullet}$ so that the following is true:
\begin{equation}
    \sum_j c_j f_j(\vx) + c_{\sbullet} \vw^\t_{\sbullet}\vx = \vzero, \quad \forall \vx \in R
\end{equation}
and for node $j$, $\partial E_j \cap R \neq \emptyset$, then $c_j = 0$.
\end{lemma}
\begin{proof}
We can apply the same logic as Lemma~\ref{lemma:colinear} to the region $R$ (Fig.~\ref{fig:proof1}(b)). For any node $j$, since its boundary $\partial E_j$ is in $R$, we can find a similar $\vx_0$ so that $\vx_0\in \partial E_j \cap R$ and $\vx_0 \notin \partial E_{j'}$ for any $j'\neq j$. We construct $B_{\vx_0, \epsilon}$. Since $R$ is an open set, we can always find $\epsilon > 0$ so that $B_{\vx_0, \epsilon} \subseteq R$ and no other boundary is in this $\epsilon$-ball. Following similar logic of Lemma~\ref{lemma:colinear}, $c_j = 0$. 
\end{proof}

\subsection{Theorem~\ref{thm:2-layer}}
\begin{proof}
In this situation, because $D_2(\vx) = D^*_2(\vx) = I$, according to Eqn.~\ref{eq:v-recursive}, $V_1(\vx) = W_1^\t$ and $V^*_1(\vx) = W_1^{*\t}$ are independent of input $\vx$. Therefore, both $A_1$ and $B_1$ are independent of input $\vx$. 

From Assumption~\ref{assumption:noise-free}, since $\rho(\vx) > 0$ in $R_0$, from Lemma~\ref{thm:sgd-critical-point}, we know that either $\vg_1(\vx) = \vzero{}$ or $\vx = \vzero$. However, since $\vx = [\tilde\vx, 1]$ has bias term, $\vg_1(\vx) = D_1(\vx) \left[A_1\vf^*_1(\vx) - B_1\vf_1(\vx)\right] = \vzero{}$. Picking node $k$, the following holds for every node $k$ and every $\vx \in R_0 \cap E_k$:
\begin{equation}
    \valpha_k^\t \vf^*(\vx) - \vbeta_k^\t\vf(\vx) = \vzero \label{eq:sgd-critical-k}
\end{equation}
Here $\valpha^\t_k$ is the $k$-th row of $A_1$, $A_1 = [\valpha_1, \ldots, \valpha_{n_1}]^\t$ and similarly for $\vbeta^\t_k$. Note here layer index $l=1$ is omitted for brevity.

For teacher $j$, suppose it is observed by student $k$, i.e., $\partial E^*_j \cap E_k \neq \emptyset$. Given all teacher and student nodes, note that co-linearity is a equivalent relation, we could partition these nodes into disjoint groups. Suppose node $j$ is in group $s$. In Eqn.~\ref{eq:sgd-critical-k}, if we combine all coefficients in group $s$ together into one term $c_s\vw^*_j$ (with $\|\vw^*_j\| = 1$), we have:
\begin{equation}
    c_s = \alpha_{kj} -\sum_{k'\in \coli(j)} \|\vw_{k'}\| \beta_{kk'}
\end{equation}
``At most'' because from Assumption~\ref{assumption:noise-free}, all teacher weights are not co-linear. Note that $\coli(j)$ might be an empty set.  

By Assumption~\ref{assumption:noise-free}, $\partial E^*_j \cap R_0 \neq \emptyset$ and by observation property, $\partial E^*_j \cap E_k \neq \emptyset$, we know that for $R = R_0 \cap E_k$, $\partial E^*_j \cap R \neq \emptyset$. Applying Lemma~\ref{lemma:local-colinear}, we know that $c_s = 0$. Since $\alpha_{kj} \neq 0$, we know $\coli(j) \neq \emptyset$ and there exists at least one student $k'$ that is aligned with the teacher $j$.
\end{proof}

\subsection{Theorem~\ref{thm:net-zero}}
\begin{proof}
We basically apply the same logic as in Theorem~\ref{thm:2-layer}. Consider the colinear group $\coli(k)$. If for all $k' \in \coli(k)$, $\beta_{k'k'} \equiv \|\vv_{k'}\|^2 = 0$, then $\vv_{k'} = \vzero$ and the proof is complete. 

Otherwise, if there exists some student $k$ so that $\vv_{k} \neq \vzero$. By the condition, it is observed by some student node $k_o$, then with the same logic we will have
\begin{equation}
\sum_{k' \in \coli(k)} \beta_{k_o, k'} \|\vw_{k'}\| = 0
\end{equation}
which is
\begin{equation}
    \vv^\t_{k_o} \sum_{k' \in \coli(k)} \vv_{k'} \|\vw_{k'}\| = 0
\end{equation}

Since $k$ is observed by $C$ students $k_o^1, k_o^2, \ldots, k_o^J$, then we have:
\begin{equation}
    \vv^\t_{k_o^j} \sum_{k' \in \coli(k)} \vv_{k'}\|\vw_{k'}\| = 0
\end{equation}
By the condition, all the $C$ vectors $\vv^\t_{k_o^j} \in \rr^C$ are linear independent, then we know that 
\begin{equation}
    \sum_{k' \in \coli(k)} \vv_{k'} \|\vw_{k'}\| = \vzero
\end{equation}
\end{proof}

\subsection{Corollary~\ref{co:zero-contribution}}
\begin{proof}
We can write the contribution of all student nodes which are not aligned with any teacher nodes as follows:
\begin{eqnarray}
    \sum_{s}\sum_{k\in \coli(s)} \vv_k f_k(\vx) &=& \sum_{s}\sum_{k\in \coli(s)} \vv_k \|\vw_k\|\sigma({\vw'_s}^\t \vx) \\
    &=& \sum_{s} \sigma({\vw'_s}^\t \vx) \sum_{k\in \coli(s)} \vv_k \|\vw_k\|
\end{eqnarray}
where $\vw'_s$ is the unit vector that represents the common direction of the co-linear group $s$. From Theorem~\ref{thm:net-zero}, for group $s$ that is not aligned with any teacher, $\sum_{k\in \coli(s)} \vv_k \|\vw_k\| = \vzero$ and thus the net contribution is zero. 
\end{proof}

\section{Main Theorems}
\subsection{Lemma~\ref{lemma:relation}} 
\begin{lemma}[Relation between Hyperplanes]
\label{lemma:relation}
Let $\vw_j$ and $\vw_{j'}$ two distinct hyperplanes with $\|\tilde\vw_j\| = \|\tilde\vw_{j'}\| = 1$. Denote $\tilde\theta_{jj'}$ as the angle between the two vectors $\tilde\vw_j$ and $\tilde\vw_{j'}$. Then there exists $\tilde\vu_{j'} \perp \tilde\vw_j$ and $\vw_{j'}^\t \tilde\vu_{j'} = \sin\tilde\theta_{jj'}$.
\end{lemma}
\begin{proof}
Note that the projection of $\tilde\vw_{j'}$ onto $\tilde \vw_j$ is:
\begin{equation}
\tilde\vu_{j'} = \frac{1}{\sin\tilde\theta_{jj'}} P^\perp_{\tilde\vw_j} \tilde\vw_{j'}
\end{equation}
It is easy to verify that $\|\tilde\vu_{j'}\| = 1$ and $\vw_{j'}^\t \tilde\vu_{j'} = \sin\tilde\theta_{jj'}$.
\end{proof}

\begin{definition}[Alignment of $(j, j')$ by error $(M_1\epsilon, M_2\epsilon)$] 
Two nodes $j$ and $j'$ are called aligned, if their weights $\vw_j = [\tilde\vw_j, b_j]$ and $\vw_{j'} = [\tilde\vw_{j'}, b_{j'}]$ satisfy the following:
\begin{equation}
\sin\tilde\theta_{jj'} \le M_1\epsilon, \quad\quad |b_j - b_{j'}| \le M_2\epsilon
\end{equation}
\end{definition}

\begin{definition}[Constrained $\eta$-Dataset]
For a weight vector $\vw$ and $\epsilon \le \epsilon_0$, a dataset $D' = D\cap I_\vw(\epsilon)$ is called a constrained $\eta$-Dataset, if for any regular $\vw'$ with $\tilde\vw^T\tilde \vw' = 0$, we have:
\begin{equation}
    N\left[D'\cap I_{\vw'}(\epsilon)\right] \le \eta_\vw \epsilon N_{D'} + (d + 1)    
\end{equation}
Note $\eta_\vw$ is independent of $\epsilon$.
\end{definition}

Note that $D'$ is always a constrained $\eta$-dataset for sufficiently large $\eta$. 

\begin{figure}
    \centering
    \includegraphics[width=\textwidth]{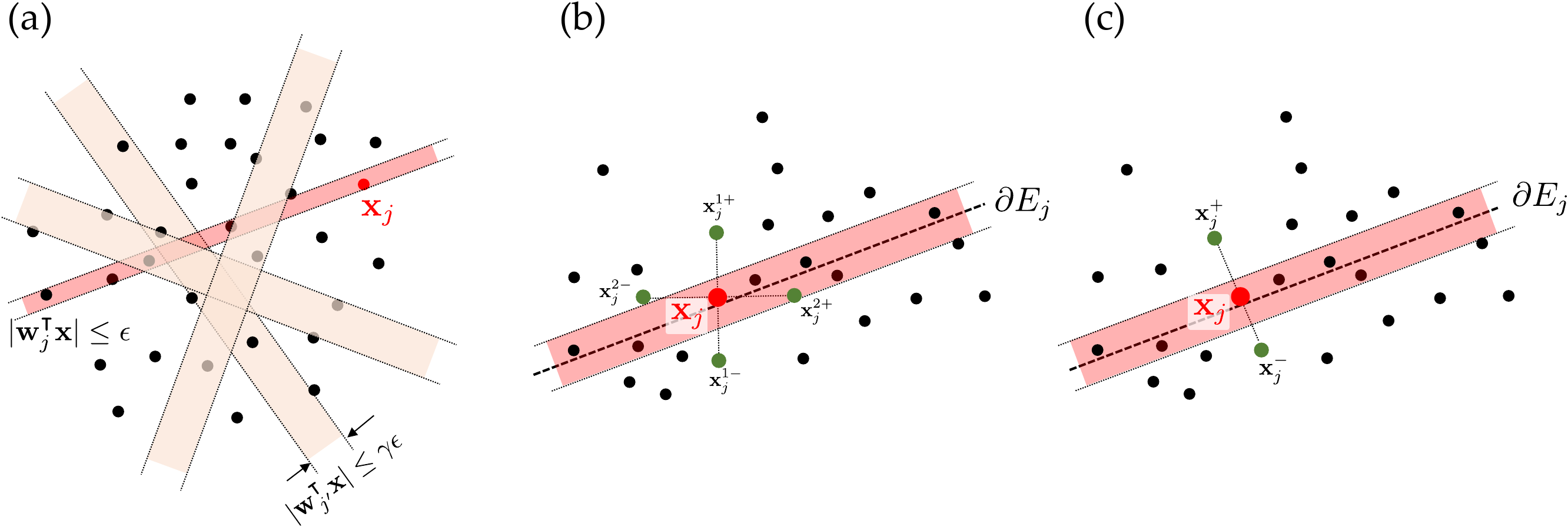}
    \caption{\textbf{(a)} Lemma~\ref{lemma:escape}: either there exists a node $j'$ aligned with node $j$, or there exists $\vx_j$ so that $|\vw_j^\t\vx| \le \epsilon$ but for all $j'\neq j$, $|\vw_{j'}^\t\vx| > \gamma\epsilon$, i.e., $\vx_j \in I_j(\epsilon) \backslash \cup_{j'\neq j} I_{j'}(\gamma\epsilon)$. 
    \textbf{(c)} Lemma~\ref{lemma:gradient-large}. For such $\vx_j$, since each data point is augmented according to teacher-agnostic augmentation (Definition~\ref{def:augmentation}), $\vx_j$ also has its own augmentation $\vx_j^{k\pm}$. Due to the property of $\vx_j$, we know $\vx_j$ and its augmentation $\vx_j^{k\pm}$ are on the same side of gradient function $h(\vx)$ and thus each node $j'\neq j$ are all linear. By checking the gradient $h$ evaluated on $\vx_j$ and its augmentation $\vx_j^{k\pm}$, we could show at least one gradient has large magnitude by contradiction. \textbf{(c)} In teacher-aware case, contradiction could follow with only 2 augmented samples, if they are constructed alone teacher $j$'s weight $\tilde \vw^*_j$.} 
    \label{fig:lemma-illustration}
\end{figure}

\def\flip{\mathrm{flip}}

\begin{lemma}
\label{lemma:escape}
Consider $K$ hyper planes, each with regular weight $\{\vw_j\}_{j=1}^K$. $\gamma > 0$ and $\epsilon > 0$ are constants. Consider one hyper-plane $j$. For a constrained $\eta$-dataset $D \subseteq I_j(\epsilon)$, if 
\begin{itemize}
    \item[(a)] $N_{D} \ge (\gamma+3) K(d+1)$.
    \item[(b)] For $j'\neq j$, $j'$ is not aligned with $j$ by error $(M_1\epsilon, M_2\epsilon)$. Here 
    \begin{equation}
        M_1 = (\gamma+3)\eta K, \quad M_2 = 1 + \gamma + (\gamma+3)\eta K\left(|b_j| + \sqrt{\frac{(3+\gamma)\mu K}{1+\gamma}} \right)
    \end{equation} 
\end{itemize}
Then there exists $\vx \in D$ so that for any $j'\neq j$, $\vx\notin I_{j'}(\gamma\epsilon)$.
\end{lemma}
\begin{proof}
Without loss of generality, we assume any angle $\tilde\theta_{jj'} \in [0, \pi/2]$. If not, we can always flip the hyper plane by sending $\vw = [\tilde\vw, b]$ to $\mathrm{flip}(\vw) = [-\tilde\vw, b]$. This gives $I_\vw(\epsilon) = I_{\flip(\vw)}(\epsilon)$ and keep the definition of $\epsilon$-alignment: $(\vw_1, \vw_2)$ is $\epsilon$-aligned if and only if $(\vw_1, \flip(\vw_2))$ is $\epsilon$-aligned, due to the fact that $\sin\theta = \sin(\pi - \theta)$.

For any $j'\neq j$, since $j'$ is not aligned with $j$ by error $(M_1\epsilon, M_2\epsilon)$, we know that either of the two cases hold.
\begin{itemize}
    \item[1)] $\sin\tilde\theta_{jj'} > M_1\epsilon$.
    \item[2)] $\sin\tilde\theta_{jj'} \le M_1\epsilon$ but $|b_j-b_{j'}| > M_2\epsilon$.
\end{itemize}
\textbf{Case 1}: By Lemma~\ref{lemma:relation} we know that there exists $\tilde\vu_{j'} \perp \vw_j$ so that 
\begin{equation}
    \tilde\vw_{j'} = \cos\tilde\theta_{jj'} \tilde\vw_j + \sin\tilde\theta_{jj'} \tilde\vu_{j'}
\end{equation}
Let $b_u = (b_{j'} - b_j \cos\tilde\theta_{jj'}) / \sin\tilde\theta_{jj'}$ and $\vu_{j'} = [\tilde\vu_{j'}, b_u]$. Then we have:
\begin{equation}
    \vw_{j'} = \cos\tilde\theta_{jj'}\vw_j + \sin\tilde\theta_{jj'}\vu_{j'} \label{eq:w-tilde}
\end{equation}
Notice that we have the following fact: if $\vx \in D \cap I_{j'}(\gamma\epsilon) \subseteq I_j(\epsilon) \cap I_{j'}(\gamma\epsilon)$, then 
\begin{equation}
    |\vu_{j'}^\t\vx| \le \frac{1}{\sin\tilde\theta_{jj'}}\left[|\vw_{j'}^\t\vx| + |\cos\tilde\theta_{jj'}||\vw_j^\t\vx|\right] \le \frac{(1+\gamma)\epsilon}{M_1\epsilon} = \frac{1+\gamma}{(3+\gamma)\eta K}
\end{equation}
Therefore, by the definition of $\eta$-dataset, we have: 
\begin{equation}
    N_{D}\left[I_{j'}(\gamma\epsilon)\right] \le N_{D}\left[|\vu_{j'}^\t\vx| \le \frac{1+\gamma}{(3+\gamma)\eta K}\right] \le \frac{1+\gamma}{(3+\gamma)K} N_{D} + (d + 1)
\end{equation}

\textbf{Case 2}: Notice the following fact: if $\vx \in D \cap I_{j'}(\gamma\epsilon) \subseteq I_j(\epsilon) \cap I_{j'}(\gamma\epsilon)$, then from Eqn.~\ref{eq:w-tilde} we know:
\begin{equation}
    \tilde\vw_{j'}^\t\tilde\vx = \cos\tilde\theta_{jj'} \tilde\vw_j^\t\tilde\vx + \sin\tilde\theta_{jj'} \tilde\vu_{j'}^\t\tilde\vx
\end{equation}
which means that
\begin{equation}
    \vw_{j'}^\t\vx - b_{j'} = \cos\tilde\theta_{jj'} (\vw_j^\t\vx - b_j) + \sin\tilde\theta_{jj'} \tilde\vu_{j'}^\t\tilde\vx
\end{equation}
Therefore, we have:
\begin{eqnarray}
    |\tilde\vu_{j'}^\t\tilde\vx| &=& \frac{1}{\sin\tilde\theta_{jj'}}\Big|\vw_{j'}^\t\vx - b_{j'} - \cos\tilde\theta_{jj'} ( \vw_j^\t\vx - b_j ) \Big| \\
    &\ge& \frac{1}{\sin\tilde\theta_{jj'}} \left[ |b_j - b_{j'}| - (1 - \cos\tilde\theta_{jj'}) |b_j| - |\vw_{j'}^\t\vx| - |\vw_j^\t\vx| \right]
\end{eqnarray}
Note that since $\tilde\theta_{jj'}\in [0, \pi/2]$, we have $1 - \cos\tilde\theta_{jj'} \le 1 - \cos^2\tilde\theta_{jj'} = \sin^2\tilde\theta_{jj'} \le \sin\tilde\theta_{jj'} \le M_1\epsilon$. Therefore, 
\begin{equation}
    |\tilde\vu_{j'}^\t\tilde\vx| \ge \frac{M_2 - 1 - \gamma}{M_1} - |b_j| = \sqrt{\frac{(3+\gamma)\mu K}{1+\gamma}} 
\end{equation}

Therefore, we have:
\begin{equation}
    N_D[I_{j'}(\gamma\epsilon)] \le N_D\left[|\tilde\vu_{j'}^\t\tilde\vx| \ge \sqrt{\frac{(3+\gamma)\mu K}{1+\gamma}}\right] \le \frac{1+\gamma}{(3+\gamma)K} N_{D_j}
\end{equation}

Combining the two cases, since $N_{D_j} \ge (3+\gamma)K(d+1)$, we know that 
\begin{eqnarray}
    \sum_{j'\neq j} N_D\left[I_{j'}(\gamma\epsilon)\right] &\le& \frac{1+\gamma}{3+\gamma}\frac{K-1}{K} N_D + (K-1)(d+1) \\
    &<& \frac{1+\gamma}{3+\gamma} N_{D_j} + \frac{1}{3+\gamma} N_{D_j} \\
    &=& \left(1 - \frac{1}{3+\gamma}\right) N_D < N_D
\end{eqnarray}
Moreover, since $N_D / (3+\gamma) \ge K(d+1) \ge 1$, so there exists at least one $\vx \in D$ so that $\vx$ doesn't fall into any bucket $I_{j'}(\gamma\epsilon)$. This means that for any $j'\neq j$, $\vx\notin I_{j'}(\gamma\epsilon)$ and the proof is complete.
\end{proof}

\textbf{Remark}. Note that the constant $|b_j|$ in $M_2$ can be smaller since we could always use a stronger bound $1 - \cos\theta \le 1 - \cos^2\theta = \sin^2\theta$ and as a result, $M_2$ would contain $|b_j|\epsilon$. For small $\epsilon$, this term is negligible.  
\begin{lemma}
\label{lemma:gradient-large}
Define 
\begin{equation}
    h(\vx) = \sum_{j'=1}^K c_{j'} \sigma(\vw_{j'}^\t\vx) + c_{\sbullet} \vw^\t_{\sbullet} \vx \label{eq:gradient}
\end{equation}
Suppose there exists $\vx_j \in I_j(\epsilon_l)\backslash \cup_{j' \neq j} I_{j'}(\epsilon_h)$ with $\epsilon_l < \epsilon_h$ and there exists a vector $\tilde \ve$ so that $\epsilon_l < \epsilon_0 \le \vw_j^\t\tilde \ve \le \epsilon_h$. Construct $2$ datapoints $\vx_j^\pm = \vx_j \pm \tilde\ve$ and set $D = \{\vx_j, \vx^\pm_j\}$. Then there exists $\vx\in D$ so that $|h(\vx)| > \frac{|c_j|}{5}(\epsilon_0 - \epsilon_l)$. 
\end{lemma}
\begin{proof}
We show that the three points $\vx_j$ and $\vx_j^\pm$ are on the same side of $\partial E_{j'}$ for any $j'\neq j$. This can be achieved by checking whether $(\vw_{j'}^\t \vx_j)(\vw_{j'}^\t \vx_j^\pm) \ge 0$ (Fig.~\ref{fig:lemma-illustration}(b) and (c)):
\begin{eqnarray}
(\vw_{j'}^\t \vx_j)(\vw_{j'}^\t \vx_j^\pm) &=& (\vw_{j'}^\t \vx_j)\left[\vw_{j'}^\t (\vx_j \pm \tilde\ve)\right] \\
&=& (\vw_{j'}^\t \vx_j)^2 \pm (\vw_{j'}^\t \vx_j) \vw_{j'}^\t\tilde\ve \\
&=& |\vw_{j'}^\t\vx_j| (|\vw_{j'}^\t\vx_j| \pm \vw_{j'}^\t\tilde\ve)
\end{eqnarray}
Since $\vx_j \notin I_{j'}(\epsilon_h)$ and $\vw_{j'}^\t\tilde\ve \le \epsilon_h$, we have:
\begin{equation}
    |\vw_{j'}^\t\vx_j| \pm \vw_{j'}^\t\tilde\ve > \epsilon_h \pm \epsilon_h \ge 0
\end{equation}
Therefore, $(\vw_{j'}^\t \vx_j)(\vw_{j'}^\t \vx_j^\pm) \ge 0$ and the three points $\vx_j$ and $\vx_j^\pm$ are on the same side of $\partial E_{j'}$ for any $j' \neq j$. 

If the conclusion is not true, then consider $h(\vx_j^+) + h(\vx_j^-) - 2h(\vx_j)$. Since $\vx_j^+ + \vx_j^- = 2\vx_j$, we know that all terms related to $\vw_{\sbullet}$ and $\vw_{j'}$ with $j\neq j$ will cancel out, due to the fact that they are in the same side of the boundary $\partial E_{j'}$ and thus behave linearly. Therefore,
\begin{equation}
    h(\vx_j^+) + h(\vx_j^-) - 2h(\vx_j) = c_j\left[\sigma(\vw_j^\t\vx_j^+) + \sigma(\vw_j^\t\vx_j^-) - 2\sigma(\vw_j^\t\vx_j)\right] 
\end{equation}
Since $|\vw_j^\t\vx_j| \le \epsilon_l$ and $\vw_j^\t\tilde\ve \ge \epsilon_0 > \epsilon_l$, it is always the case that:
\begin{itemize}
\item[(a)] $\vw_j^\t\vx^+_j = \vw_j^\t\vx_j + \vw_j^\t\tilde\ve > 0$ and by ReLU properties $\sigma(\vw_j^\t\vx^+_j) = \vw_j^\t\vx^+_j$. 
\item[(b)] $\vw_j^\t\vx^-_j = \vw_j^\t\vx_j - \vw_j^\t\tilde\ve < 0$ so by ReLU properties $\sigma(\vw_j^\t\vx^-_j) = 0$.
\end{itemize}
Therefore if $\vw_j^\t\vx_j \ge 0$ then:
\begin{eqnarray}
    |h(\vx_j^+) + h(\vx_j^-) - 2h(\vx_j)| &=& |c_j\vw_j^\t (\vx^+_j - 2\vx_j)| = |c_j\vw_j^\t (\vx_j + \tilde\ve - 2\vx_j)| \\
    &=& |c_j\vw_j^\t (\tilde\ve - \vx_j)| \ge |c_j (\epsilon_0 - \vw_j^\t\vx_j)| \\
    &\ge& |c_j|(\epsilon_0 - \epsilon_l)
\end{eqnarray}
if $\vw_j^\t\vx_j < 0$ then:
\begin{equation}
    |h(\vx_j^+) + h(\vx_j^-) - 2h(\vx_j)| = |c_j\vw_j^\t \vx^+_j| = |c_j\vw_j^\t(\vx_j + \tilde\ve)| \ge |c_j|(\epsilon_0 - \epsilon_l)
\end{equation}

On the other hand, from gradient condition, we have: 
\begin{equation}
    |h(\vx_j^+) + h(\vx_j^-) - 2h(\vx_j)| \le \frac{4}{5}|c_j|(\epsilon_0 - \epsilon_l)
\end{equation}
which is a contradiction. 
\end{proof}
For Leaky ReLU, the proof is similar except that the final condition has an additional $1-\cleaky$ factor.

\begin{assumption}
\label{assumption:main-appendix}
(a) Two teacher nodes $j\neq j'$ are not $\epsilon_0$-aligned. (b) The boundary band $I_j(\epsilon)$ of each teacher $j$ overlaps with the dataset: 
\begin{equation}
N_D[I_j(\epsilon)] \ge \tau \epsilon N_D 
\end{equation}
\end{assumption}

\begin{theorem}[Two-layer Specialization with Polynomial Samples]
\label{thm:finite-sample-recovery-appendix}
Let $K = m_1 + n_1$. For $0 < \epsilon \le \epsilon_0$, for any finite dataset $D$ with $N = \Theta( cK^{5/2}d^2\tau^{-1}\epsilon^{-1}\kappa^{-1})$, for any teacher satisfying Assumption~\ref{assumption:main} and student trained on $D' = \aug(D)$ whose weight $\hW$ satisfies:
\begin{itemize}
    \item[(1)] For $\epsilon\in[0, \epsilon_0]$, the hyperplane band $I_j(\epsilon)$ of a teacher is observed by a student node $k$: $N_D[I_j(\epsilon) \cap E_k] \ge \kappa N_D[I_j(\epsilon)]$;
    \item[(2)] Small gradient: $\|\vg_1(\vx, \hW)\|_\infty \le \frac{\alpha_{kj}}{5cK^{3/2}\sqrt{d}} \epsilon,\,\vx\in D'$,
\end{itemize}
where $c > 0$ is a constant related to dataset properties. Then there exists a student $k'$ so that $(j, k')$ is $\epsilon$-aligned.
\end{theorem}
\begin{proof}
Let $K = m_1 + n_1$. Let $\epsilon' = \epsilon / c K^{3/2}\sqrt{d}$ with some constant $c = \Theta(\max_j \mu_{\vw^*_j} \max_j \eta_{\vw^*_j}) > 0$. 

We construct a basic dataset $D$ with $N = \Theta(cK^{5/2}d^2\epsilon^{-1}\tau^{-1}\kappa^{-1})$ samples and use the augmentation operator $\aug(D)$:
\begin{equation}
    \aug(D) = \{\vx^\pm_k = \vx \pm 2\epsilon\tilde\ve_u / cK^{3/2},\ \ \vx \in D,\ \ u = 1,\ldots,d\} \cup D
\end{equation}
where $\tilde\ve_k$ is axis-aligned unit directions with $\|\tilde\ve_k\| = 1$. It is clear that $|\aug(D)| = (2d+1)|D|$. Let $D' = \aug(D)$.
    
Let $j$ be some teacher node. Consider a slice of basic dataset $D \cap I_j(\epsilon')$. By Assumption~\ref{assumption:main-appendix}(b), we know that 
\begin{equation}
    N_D[I_j(\epsilon')] = N[D \cap I_j(\epsilon')] \ge \tau \epsilon' N_D = \cO(K d^{3/2}\kappa^{-1})
\end{equation}
and thus for $N_D[I_j(\epsilon') \cap E_k]$, we know that  
\begin{equation}
    N_D[I_j(\epsilon') \cap E_k] \ge \kappa N_D[I_j(\epsilon')] = \cO(K d^{3/2})
\end{equation}
With a sufficiently large constant in $N$, we have $N_D[I_j(\epsilon') \cap E_k] \ge (\gamma + 3)K(d+1)$ with $\gamma = 2\sqrt{d}$. We then apply Lemma~\ref{lemma:escape}, which leads the two following cases:

\textbf{Case 1}. There exists weight $k'$ so that the following alignment condition holds:
\begin{equation}
    \sin\tilde\theta_{jk'} \le M'_{1j}\epsilon' = M_{1j}\epsilon, \quad |b^*_j - b_{k'}| \le M'_{2j}\epsilon' = M_{2j}\epsilon
\end{equation}
Where $M'_{1j} = \Theta(K\gamma)$ and $M'_{2j} = \Theta(K^{3/2}\gamma)$. Therefore, $M_{1j} = \Theta(K\gamma/cK^{3/2}\sqrt{d}) = o(1)$ and $M_{2j} = \Theta(1)$. Choosing the constant $c > 0$ so that we have $M_{1j} \le 1$ and $M_{2j} \le 1$ and thus 
\begin{equation}
    \sin\tilde\theta_{jk'} \le \epsilon, \quad |b^*_j - b_{k'}| \le \epsilon,
\end{equation}
which means that $(j, k')$ are $\epsilon$-aligned. Note that no other teacher is $\epsilon$-aligned with $j$. So $k'$ has to be a student node and the proof is complete. 

\textbf{Case 2}. If the alignment condition doesn't hold, then according to Lemma~\ref{lemma:escape}, there exists $\vx_j$ so that (note that $\gamma = 2\sqrt{d}$):
\begin{equation}
\vx_j \in I_j(\epsilon') \backslash \cup_{j'\neq j} I_{j'}(2\sqrt{d}\epsilon'). \label{eq:exist-x}
\end{equation}
Here all $j'$ includes all teacher and student nodes (a total of $K$ nodes), excluding the current node $j$ under consideration. Since $\{\tilde\ve_\vu\}^d_{u=1}$ forms orthonormal bases, there exists at least one $u$ so that $1 \ge \vw_j^{*\t}\tilde\ve_u \ge 1/\sqrt{d}$ (with proper sign flipping of $\tilde\ve_u$). Let $\tilde\ve = 2\epsilon\tilde\ve_u / cK^{3/2} = 2\epsilon'\sqrt{d}\tilde\ve_u$, we have $2\epsilon' \le \vw_j^{*\t}\tilde\ve \le 2\sqrt{d}\epsilon'$. Applying Lemma~\ref{lemma:gradient-large} and we know the additional samples required in the lemma is already in $\aug(D)$. Therefore, that there exists $\vx \in \aug(D_j) \subset D$ so that $|h(\vx)| > \frac{|\alpha_{kj}|}{5}(2\epsilon' - \epsilon') = \frac{|\alpha_{kj}|}{5cK^{3/2}\sqrt{d}}\epsilon$, which is a contradiction. 
\end{proof}

\begin{theorem}[Two-layer Specialization with Teacher-aware Dataset with Polynomial Samples] 
\label{thm:finite-sample-teacher-aware-recovery-appendix}
For $0 < \epsilon \le \epsilon_0$, for any finite dataset $D$ with $N = \Theta(cK^{5/2} d\tau^{-1}\epsilon^{-1})$, given a teacher network $\cW^*$ satisfying Assumption~\ref{assumption:main} and student trained on $D' = \aug(D, \cW^*)$ whose weight $\hW$ satisfies
\begin{itemize}
    \item[(1)] For $\epsilon\in[0, \epsilon_0]$, the hyperplane band $I_j(\epsilon)$ of a teacher is observed by a student node $k$: $N_D[I_j(\epsilon) \cap E_k] \ge \kappa N_D[I_j(\epsilon)]$;
    \item[(2)] Small gradient: $\|\vg_1(\vx, \hW)\|_\infty \le \frac{\alpha_{kj}}{5cK^{3/2}} \epsilon$, for $\vx \in D'$,
\end{itemize}
then there exists a student $k'$ so that $(j, k')$ is $\epsilon$-aligned.
\end{theorem}
\begin{proof}
Pick $\epsilon' = \epsilon / cK^{3/2}$ ($c$ defined as in Theorem~\ref{thm:finite-sample-recovery-appendix}). For dataset $D$, by Assumption~\ref{assumption:main}, we know that for $D \cap I_j(\epsilon') \cap E_k$:
\begin{equation}
    N[D \cap I_j(\epsilon') \cap E_k] = N_D[I_j(\epsilon')\cap E_k] \ge \kappa N_D[I_j(\epsilon')] \ge \kappa \tau \epsilon' N_D = \cO(Kd)
\end{equation}

Apply Lemma~\ref{lemma:escape} with $\gamma = 2$ and similarly we know that either there exists a student $k'$ so that $(j, k')$ is $\epsilon$-align (with $M_{1j} \le 1$ and $M_{2j} \le 1$), or there exists $\vx_j$ such that 
\begin{equation}
    \vx_j \in I_j(\epsilon') \backslash \cup_{j'\neq j} I_{j'}(2\epsilon').
\end{equation}
Setting $\tilde\ve = 2\epsilon'\vw^*_j$ and we know that $\vw_j^{*\t}\tilde\ve = 2\epsilon'$.  Since we have used teacher-aware augmentation, applying Lemma~\ref{lemma:gradient-large} with $\epsilon_h = \epsilon_0 =  2\epsilon'$ and $\epsilon_l = \epsilon'$, the conclusion follows. 
\end{proof}

\subsection{Theorem~\ref{thm:multi-layer}}
\begin{proof}
The proof is similar to 2-layer case (Theorem~\ref{thm:finite-sample-recovery} in the main text or Theorem~\ref{thm:finite-sample-recovery-appendix} in Appendix). The only difference is that instead of thinking about $K_1 = m_1 + n_1$ boundaries, we need to think about all the $Q$ boundaries introduced by the top-level and obtain a data point $\vx_j$ so that it is within the boundary of node $j$, but far away from all other possible boundaries:  
\begin{equation}
    \vx_j \in I_j(\epsilon') \backslash \cup_{j'\neq j} I_{j'}(2\sqrt{d}\epsilon').
\end{equation}
Where $j'$ includes all the boundaries induced. This could be exponential. Note that for each intermediate node $j'$, its boundary $\vw_{j'}\vf = 0$ will be ``bent'' whenever the underlying feature $\vf$, which is the output of a set of ReLU / Leaky ReLU nodes, has shifted their activation patterns. 
\end{proof}

\begin{figure}
    \centering
    \includegraphics[width=0.8\textwidth]{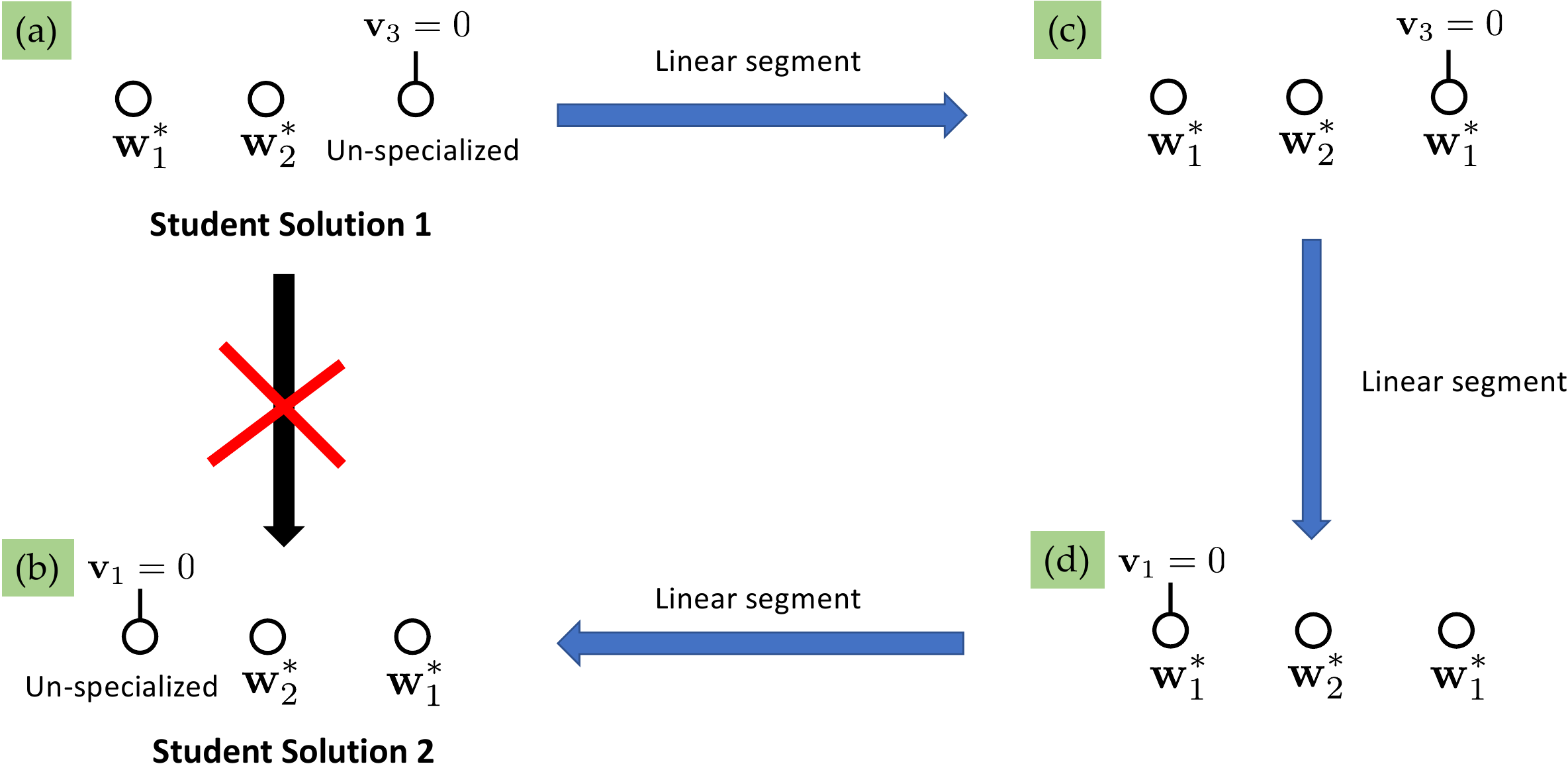}
    \caption{A piece-wise linear curve between two lost-cost student solutions.}
    \label{fig:connectivity}
\end{figure}

\section{Connectivity}
\label{sec:connectivity}
We construct a piece-wise linear curve from two low-cost student solutions as in Fig.~\ref{fig:connectivity}. Consider two student networks with 3 hidden nodes trained with the same teacher with 2 nodes. Once they converge, assuming their gradients are zero, there could be many different ways of specializations that satisfy Theorem~\ref{thm:2-layer}. 

Fig.~\ref{fig:connectivity}(a) and (b) show two such specializations $\cW^{(1)}$ and $\cW^{(2)}$. For $\cW^{(1)}$, $\vw^{(1)}_1 = \vw^*_1$, $\vw^{(1)}_2 = \vw^*_2$ and $\vw^{(1)}_3$ is an un-specialized node whose fan-out weights are zero ($\vv^{(1)}_3 = \vzero{}$). For $\cW^{(2)}$, $\vw^{(2)}_3 = \vw^*_1$, $\vw^{(2)}_2 = \vw^*_2$ and $\vw^{(2)}_1$ is an un-specialized node whose fan-out weights are zero ($\vv^{(2)}_1 = \vzero{}$). 

If we directly connect these two solutions using a straight line, the intermediate solution will be no longer low-cost since a linear combination $\lambda \vw^{(1)}_1 + (1-\lambda) \vw^{(2)}_1 = \lambda \vw^*_1 + (1-\lambda) \vw^{(2)}_1$ can be a random (un-specialized) vector, and its corresponding fan-out weights $\lambda \vv^{(1)}_1 + (1-\lambda) \vv^{(2)}_1 = \lambda \vv^{(1)}_1$ is also non-zero. This yields a high-cost solution. 

On the other hand, if we take a piece-wise linear path (a)-(c)-(d)-(b), then each line segment will have low-cost and we move $\vw^*_1$ from node 1 to node 3. We list the line segment construction as follows:
\begin{itemize}
    \item Start from $\cW^{(1)}$.
    \item \textbf{(a)-(c)}. Blend $\vw^*_1$ into an un-specialized weight: $\vw_3(t) = (1-t)\vw^{(1)}_3 + t \vw^*_1$. This won't change the output since $\vv_3^{(1)} = \vzero{}$. 
    \item \textbf{(c)-(d)}. Move $\vv^{(1)}_1$ from node 1 to node 3: 
    \begin{eqnarray}
        \vv_3(t) &=& (1-t) \vv_3^{(1)} + t \vv^{(1)}_1 = t \vv^{(1)}_1 \\
        \vv_1(t) &=& (1-t) \vv^{(1)}_1 + t \vv_3^{(1)} = (1 - t) \vv^{(1)}_1
    \end{eqnarray} 
    This won't change the output since $\vv_1(t) + \vv_3(t) = \vv^{(1)}_1$ and their weights are both $\vw^*_1$.
    \item \textbf{(d)-(b)}. Change $\vw_1$ to be the unspecified weight in $\cW^{(2)}$: $\vw_1(t) = (1 - t)\vw_1^* + t \vw^{(1)}_2$. This won't change the output since now $\vv_1 = \vzero{}$.
    \item Arrive at $\cW^{(2)}$.
\end{itemize}

\section{Empirical results}
\label{sec:teacher-construction}
We construct teacher networks in the following manner. For two-layered network, the output dimension $C = 50$ and input dimension $d = m_0 = n_0 = 100$. For multi-layered network, we use 50-75-100-125 (i.e, $m_1 = 50, m_2 = 75, m_3 = 100, m_4 = 125$, $L = 5$, $d=m_0=n_0=100$ and $C=m_5=n_5=50$). The teacher network is constructed to satisfy Assumption~\ref{assumption:main}: at each layer, teacher filters are distinct from each other and their bias is set so that $\sim 50\%$ of the input data activate the nodes, maximizing the number of samples near the boundary. 

We generate the input distribution using $\cN(0, \sigma^2 I)$ with $\sigma = 10$. For all 2-layer experiments, we sample $10000$ as training and another $10000$ as evaluation. We also tried other distribution (e.g., uniform distribution $U[-1, 1]$), and the results are similar. 

\section{Code Release} 
We have attached all Python codes in our supplementary material submission. 

\fi

\end{document}